\documentclass{article}

\usepackage[utf8]{inputenc} % allow utf-8 input
\usepackage[T1]{fontenc}    % use 8-bit T1 fonts
\usepackage{hyperref}       % hyperlinks
\usepackage{url}            % simple URL typesetting
\usepackage{booktabs}       % professional-quality tables
\usepackage{amsfonts}       % blackboard math symbols
\usepackage{nicefrac}       % compact symbols for 1/2, etc.
\usepackage{microtype}      % microtypography

\usepackage{wrapfig}

\usepackage{bbm}
\usepackage{amssymb}
\usepackage{amsthm}
\usepackage{mathtools}
\usepackage{natbib}
\usepackage[normalem]{ulem}
 \usepackage[]{algorithm2e}
\usepackage{bbm}
\usepackage{graphicx}

\usepackage{enumerate}
\usepackage[backgroundcolor=White,textwidth=0.8in]{todonotes}
\newtheorem{proposition}{Proposition}
\newtheorem{theorem}{Theorem}
\newtheorem{lemma}{Lemma}

\theoremstyle{definition}
\newtheorem{remark}{Remark}
\newcommand{\X}{\mathcal{X}}
\newcommand{\Y}{\mathcal{Y}}

\newcommand{\Gd}{G_{\Delta}}
\newcommand{\Gds}{G_{\Delta}^*}

\newcommand{\gmo}{g_{-1}} % g minus one
\newcommand{\gD}{g_{\Delta}} % g Delta

\DeclareMathOperator*{\argmin}{arg\,min}

\newcommand{\lp}{\left(}
\newcommand{\rp}{\right)}
\newcommand{\lb}{\left [}
\newcommand{\rb}{\right]}
\newcommand{\lc}{\left\{}
\newcommand{\rc}{\right\}}
\newcommand{\gd}{\gamma_{\delta}}  
\newcommand{\indi}{\mathbbm{1}}
\newcommand{\Oh}{\mathcal{O}}

\newcommand{\mbb}[1]{\mathbb{#1}}

\newenvironment{proofoutline}
 {\proof[Proof outline]}
 {\endproof}

\usepackage{microtype}
\usepackage{graphicx}
\usepackage{subfigure}
\usepackage{booktabs} % for professional tables

\usepackage{hyperref}

\title{Binary Classification with Bounded Abstention Rate}

\author{
Shubhanshu Shekhar\\
  \texttt{shshekha@eng.ucsd.edu}
  \and
Mohammad Ghavamzadeh\\
  \texttt{mohammad.ghavamzadeh.inria.fr}
  \and 
  Tara Javidi\\
  \texttt{tjavidi@eng.ucsd.edu}
}

\date{}
\begin{document}

\maketitle

\begin{abstract}
% This paper considers the problem of binary classification in which the learner has an additional option of abstaining from declaring a label. Unlike most existing works in the literature which assign a fixed cost for abstaining, the learner is allowed to employ the abstain option up to a fraction $\delta$ of input samples without incurring any cost. We obtain a complete characterization of the Bayes optimal classifier for this problem, generalizing as well as providing an alternative proof of the existing results. exploiting the particular threshold type structure of the optimal classifier, we then construct a plug-in abstaining classifier that uses the unlabelled samples to approximate the thresholds of an estimator of the regression function. The constructed classifier adapts to the unknown smoothness of the regression function and can achieve \emph{fast}, i.e., faster than $\tilde{\mathcal{O}}\lp n^{-1/2} \rp$,  rates of convergence under some margin and detectability conditions.  We then propose a computationally feasible algorithm, which builds upon the prior work on convex loss surrogates designed for the problem of classification with fixed cost of rejection, and obtain bounds on its excess risk in the \emph{realizable} case. We also propose a naive baseline algorithm for comparison and  empirically validate the performance of these algorithms on some standard machine learning benchmark datasets.

We consider the problem of binary classification with abstention in the relatively less studied \emph{bounded-rate} setting. We begin by obtaining a characterization of the Bayes optimal classifier for an arbitrary input-label distribution $P_{XY}$. Our result generalizes and provides an alternative proof for the result first obtained by \cite{chow1957optimum}, and then re-derived by \citet{denis2015consistency}, under a continuity assumption on $P_{XY}$. 
%removing a continuity assumption required by \cite{chow1957optimum}. 
We then propose a plug-in classifier  that employs unlabelled samples to decide the region of abstention and  derive an upper-bound on the excess risk of our classifier under standard \emph{H\"older smoothness}  and \emph{margin} assumptions. Unlike the plug-in rule of \citet{denis2015consistency},  our constructed classifier satisfies the abstention constraint with high probability and  can also deal with discontinuities in the empirical cdf. We also derive  lower-bounds that demonstrate the minimax  near-optimality of our proposed algorithm. To address the excessive complexity of the plug-in classifier in high dimensions, we propose a computationally efficient algorithm that builds upon prior work on convex loss surrogates, and obtain bounds on its excess risk in the \emph{realizable} case. We empirically compare the performance of the proposed algorithm with a baseline on a number of UCI benchmark datasets.

% This document provides a basic paper template and submission guidelines.
% Abstracts must be a single paragraph, ideally between 4--6 sentences long.
% Gross violations will trigger corrections at the camera-ready phase.
\end{abstract}

%--------------------------------------------------------- 
%--------------------------------------------------------- 
% INTRODUCTION 
%--------------------------------------------------------- 
%--------------------------------------------------------- 

%\todo{I think we're using ICML-18 style. We shall download and use the ICML-19 style.}

\section{Introduction}
\vspace{-1em}
We consider the problem of binary classification with a caveat that the classifier has an additional option to \emph{abstain}, or not declare the label, for some points in the input space. This framework, alternatively referred to as \emph{classification with rejection}~\citep{cortes2016learning} or \emph{selective classification}~\citep{elyaniv2010selective}, allows the learner an option to withhold her decision over the highly noisy parts of the input space, in which the probability of making an error is large. Classification with abstention provides a suitable framework for modeling several practical scenarios. One example of such a problem is automated medical diagnosis systems, where the consequences  of a wrong diagnosis may be much more critical than the alternative of the subject having to  undergo more tests. Other relevant areas of applications include DNA sequencing, dialog systems, and detecting harmful contents on the internet. 

The most common approach to learning with abstention is the fixed-cost setting, in which the classifier incurs a fixed cost every time the abstain option is invoked. Recent works exploring different aspects of this approach include~\citet{cortes2016learning,wegkamp2011support,bartlett2008classification,herbei2006classification}. In this paper, we consider the relatively less studied formulation of this problem in which the learner is allowed to abstain for up to a fixed fraction $\delta$ of the input samples without incurring any costs. This formulation models situations where we cannot assign a precise cost to abstention but  the  bottleneck is the rate at which the abstained inputs are processed \citep{pietraszek2007use}. 

%----https://www.overleaf.com/project/5ce1f40cc6fe8e13bf93802f---------------------------------------------------------------------------
% OUTLINE OF THE PAPER
%-------------------------------------------------------------------------------

% The rest of the paper is organized as follows. We provide a summary of the related prior work in Section~\ref{subsec:prior_work} and present our contributions in Section~\ref{subsec:contributions}. We formally state the problem and assumptions in Section~\ref{sec:problem_setup}. We derive the Bayes optimal classifier in Section~\ref{sec:bayes_optimal}, and construct and analyze a plug-in classifier in Section~\ref{sec:plugin_classifier}. In Section~\ref{sec:convex}, we propose computationally feasible algorithms using convex surrogates for the problem. We empirically evaluate our algorithms on two standard datasets in Section~\ref{sec:experiments} and conclude the paper in Section~\ref{sec:conclusion}.% contains the conclusion.\todo{We shall mention the experiments here.} 

%--------------------------------------------------------- 
% PRIOR WORK
%--------------------------------------------------------- 

% \subsection{Prior work} 
% \label{subsec:prior_work}
\paragraph{Prior Work.} 
The formal analysis of the problem of classification with a reject option was initiated by~\citet{chow1957optimum,Chow1970:Optimum}.~\citet{chow1957optimum} derived the Bayes optimal classifier for this problem considering the fixed-cost abstention model, as well as under the bounded-rate of abstention constraint. In the latter case, some continuity assumptions were implicitly made on the joint distribution, which we relax in this paper.~\citet{Chow1970:Optimum} further obtained a functional relation between the error rate and the rejection rate. 

More recent works in this area have concentrated on the fixed-cost setting, in which employing the abstain option incurs a fixed cost $\lambda \in (0,1/2)$, which is assumed to be known to the learner.~\citet{herbei2006classification} obtained convergence rates on the excess risk for plug-in and risk minimization based classifiers. ~\citet{bartlett2006convexity} introduced a convex surrogate loss, called the Generalized Hinge Loss, for this problem and proved results on its calibration and excess risk.
~\citet{yuan2010classification} further obtained necessary and sufficient  conditions for the infinite sample complexity of arbitrary convex surrogate loss functions. Other related works include~\citet{wegkamp2007lasso} and~\citet{wegkamp2011support} that analyzed the binary classification with reject option with $\ell_1$-regularization. 
\citet{cortes2016learning} considered this problem in a  new framework, in which an abstaining classifier is represented by a pair of functions $(h,r)$, where the sign of $h$ is used for prediction and the sign of $r$ decides whether to abstain or not. They proposed new calibrated convex surrogate loss functions for this problem and obtained generalization and consistency guarantees. This framework was further extended to construct boosting classifier~\citep{cortes2016boosting} as well as to the online setting~\citep{cortes2017online}. Other related work which employ a pair of functions to represent abstaining classifiers include \citep{elyaniv2010selective, wiener2011agnostic}. 

% ~\citet{elyaniv2010selective,wiener2011agnostic} studied the trade-off between the coverage and risk in a similar setting that involves learning a pair of functions $(f,g)$, where $f$ is used for prediction and $g$ is the selection function.~\citet{elyaniv2010selective} showed that \emph{perfect} learning, i.e.,~zero risk learning, can be achieved with certain guarantees on the coverage for the {\em realizable} setting. In~\citet{wiener2011agnostic}, it was shown that a similar notion, called \emph{weakly optimal} learning, was achievable in the {\em agnostic} setting. 

Unlike the fixed-cost setting, the literature is relatively sparse for the bounded-rate of abstention.~\citet{pietraszek2007use} proposed algorithms for this as well as  related settings using ROC analysis. The work of  \cite{denis2015consistency}  is closely related to the results of Section~\ref{sec:bayes_optimal} and Section~\ref{sec:plugin_classifier} of our paper. More specifically, \cite{denis2015consistency} also obtained the Bayes optimal classifier for the bounded-rate setting, and proposed a general plug-in strategy for constructing an abstaining classifier given any consistent estimator of the regression function. However, both these results in \citet{denis2015consistency} required certain continuity assumptions (\textbf{A1} and \textbf{A2} in \citep{denis2015consistency}), which we relax in our work (Remark~\ref{remark:chow_comparison} and Remark~\ref{remark:denis_hebiri_comparison}).
Furthermore, our approach in constructing the plug-in classifier is complementary to that of \citep{denis2015consistency} in the following way: instead of proposing a general strategy which takes in as input an estimator, we construct a specific estimator and and a particular randomized rule which allows us to have certain desirable properties such as tight control over abstention rate, and adaptivity to local smoothness parameters.

% Furthermore, we also derive lower bounds on the excess risk which establish the minimax (near)-optimality of our classifier. Finally, noting the computational infeasibility of the plug-in approach in higher dimensions, we then  propose and analyze a computationally efficient algorithm employing convex surrogates which is more important for practical applications. 

% \todomgh{We should list in more details the contributions of Denis both in terms of Bayes optimal and plug-in and mention how we improve over.}

%---------------------------------------------------------
% CONTRIBUTIONS 
%---------------------------------------------------------

% \subsection{Contributions}
% \label{subsec:contributions}
\paragraph{Contributions.} We now highlight the four main contributions of this paper to the problem of binary classification with bounded-rate of abstention. 
 \textbf{1)} In Section~\ref{sec:bayes_optimal}, we derive the form of the Bayes optimal classifier for this problem for arbitrary input-label joint distributions. The result extends the threshold type classifier first derived by~\citet{chow1957optimum}, and re-derived by \cite{denis2015consistency},  and provides an alternate and more comprehensive proof (see Remark~\ref{remark:chow_comparison}).
%\item 
\textbf{2)} We then propose a plug-in abstaining classifier which adapts to the unknown smoothness of the regression function in a data driven manner, and derive upper-bounds on its excess risk in terms of the number of required labelled and unlabelled samples. Unlike the plug-in classifier of \citep{denis2015consistency}, our proposed classifier satisfies the constraint with high probability and does not impose the continuity condition on the empirical cdf (see Remark~\ref{remark:denis_hebiri_comparison}). 
\textbf{3)} We also demonstrate the minimax near optimality of our classifier by deriving lower-bound on the excess risk (Theorem~\ref{theorem:lower_bound}). 
\textbf{4)} Since the implementation of the plug-in classifier may be intractable in higher dimensions, we also propose a computationally feasible algorithm that leverages the existing algorithms for the fixed cost setting, and derive bounds on its excess risk. We also propose a  baseline algorithm for comparison, which uses convex surrogates for both objective and constraints. Preliminary empirical results suggest that these algorithms can be used to learn classifiers with tight control over the rejection rate. 

% \todomgh{Improvements over Denis, especially for in the plug-in case are missing.}

%--------------------------------------------------------- 
% PROBLEM SETUP
%--------------------------------------------------------- 
\section{Problem Setup}
\label{sec:problem_setup}
\vspace{-1em}
Let $\X$ denote  the input space, and $\Y = \{-1,1\}$ denote the set of labels to be assigned to points in $\X$. For simplicity, we consider $\X = [0,1]^D$ for some $D>0$ and use $\|\cdot\|$ to represent the Euclidean norm on $\X$. The classification problem is completely specified by $P_{XY}$, the joint distribution of the input-label random variables. Equivalently, we can represent the problem in terms of the marginal over the input space, $P_X$, and the regression function $\eta(x) \coloneqq P_{Y|X}\lp Y=1 \mid X=x\rp$. 

A (randomized) abstaining classifier can be represented by a mapping  $g:\X \mapsto \mathcal{P}\lp \Y_1\rp$, where $\Y_1 =\Y \cup \{\Delta\}$, the symbol $\Delta$ represents the option of the classifier to abstain from declaring a label, and $\mathcal{P}(\Y_1 )$ represents the set of probability distributions on $\Y_1$. Such a classifier $g$ comprises of three functions $g_i:\X \to [0,1]$, for $i \in \Y_1$, satisfying $\sum_{i \in \Y_1} g_i(x) = 1$, for all $x \in \X$. A classifier $g$ is called \emph{deterministic} if the functions $g_i$ take values in the set $\{0,1\}$, for $i \in \Y_1$. Every deterministic classifier $g$ partitions the set $\X$ into three disjoint sets $(G_{-1},G_1,G_{\Delta})$ and we will use the two representations of a deterministic classifier interchangeably.
% Our goal is to construct a classifier $g:\X \mapsto \Y\cup \{\Delta \}$, where the symbol $\Delta$ represents the option of the classifier to  abstain from declaring a label. 
% Let $(X,Y)$ denote the input-label random variables. Then
% we can decompose the joint density function $f_{XY}(x,y)$ as follows: $f_{XY}(x,y) = f_X(x)\big( \eta(x)\mathbbm{1}_{y=1} + (1-\eta(x))\mathbbm{1}_{y=0} \big)$ where $\eta(x) = P(Y=1|X=x)$ is the regression function. 
% Every (deterministic) classifier $g$ partitions the set $\X$ into three disjoint sets $G_{-1}$, $G_1$, and $G_{\Delta}$, and we will use the two representations of a classifier interchangeably.
We define the misclassification risk of an abstaining classifier $g$ as 
\begin{equation*}
    R(g) \coloneqq P_{XY}\big( g(X)\neq Y\ , g(X) \neq \Delta \big). 
\end{equation*}
The classification problem with bounded rate of abstention can then be formally stated as 
\begin{equation}
\tag{$CA_{\delta}$}
\begin{aligned}
 \underset{g}{\text{$\min$}}& \hspace{0.5em} R(g), \qquad
 \text{subject to} \hspace{2em} P_X\big(g(X) = \Delta\big) \leq \delta.
\end{aligned}
\label{bayes_optimal_delta}
\end{equation}
To construct an abstaining classifier, we assume the availability of a training set of $n$ labelled samples $S_l = \{(X_j,Y_j)\ \mid \ 1\leq j\leq n\}$ and $m$ unlabelled samples $S_u = \{X_j\ \mid \ n+1 \leq j \leq n+m\}$. The unlabelled samples will be used to estimate the measure of the region in which a candidate classifier abstains. We will follow an approach analogous to that in~\citet{rigollet2011neyman,tong2013plug} and impose the requirement that the constructed classifier must satisfy the constraint in~\eqref{bayes_optimal_delta} with high probability. This is in contrast to the scheme proposed in \citep{denis2015consistency}, in which this constraint is only satisfied asymptotically. 
% \todomgh{You may want to highlight the difference between our results with Denis here again.}

%---------------------------------------------------------------------------
% ASSUMPTIONS
%---------------------------------------------------------------------------

%------------------------------------------------------------------------------
% 2.1 Assumptions 
%------------------------------------------------------------------------------

\paragraph{Assumptions.}
We now state the assumptions required for our theoretical analysis. %of the performance of classifiers with bounded abstention rate. 

\vspace{-1em}
\begin{enumerate}

% Label items of the list with their index.
\renewcommand{\labelenumi}{{\theenumi}}

% Define the enumeration index as $\mathbb{A}.1$... 
\renewcommand{\theenumi}{($\mbb{A}$.\arabic{enumi})}

%MARGIN ASSUMPTION
\item \label{assump:margin}  
The  input-label distribution $P_{XY}$ satisfies the {\em margin assumption} with parameters $C_0 > 0$ and $\rho_0 \geq 0$, for $\gamma$ in the set $\{1/2-\gd, 1/2 + \gd\}$, which means that for any $t>0$, we have $P_X\lp |\eta(X) - \gamma| \leq t\rp \leq C_0t^{\rho_0}$, for $\gamma\in\{1/2-\gd, 1/2 + \gd\}$.     

%DETECTABILITY ASSUMPTION
\item  \label{assump:detect} 
For the values of $\gamma$ in the same sets as in~\ref{assump:margin}, we define the {\em detectability condition} with parameters $C_1>0$ and $\rho_1\geq \rho_0$ as $P_X\lp |\eta(X) - \gamma| \leq t\rp \geq C_1t^{\rho_1}$, for any $t>0$. 

%HOLDER CONTINUITY OF REGRESSION FUNCTION
\item \label{assump:holder} 
The regression function $\eta$ is H\"older continuous with parameters $L>0$ and $0<\beta \leq 1$, i.e.,~for all $x_1, x_2 \in \lp \X, \|\cdot\|\rp$, we have 
$|\eta(x_1) - \eta(x_2)| \leq L\|x_1-x_2\|^{\beta}$.

%DENSITY LOWER BOUND
\item \label{assump:mu_min} 
The marginal distribution over the input space, $P_X$, has a density w.r.t.~the Lebesgue measure, which is bounded from below by $\mu_{\min}>0$. 

%DENSITY UPPER BOUND
\item \label{assump:mu_max} The marginal distribution over the input space, $P_X$, has a density w.r.t.~the Lebesgue measure, which is bounded from above by $\mu_{\max}<\infty$. 
\end{enumerate}

The \emph{margin assumption}~\ref{assump:margin} ensures that for a range of threshold values, the amount of $P_X$ measure for sets with values in the vicinity of that level is not too large.
It has been employed in prior works such as \citep{herbei2006classification, wegkamp2007lasso, bartlett2008classification}. 
The \emph{detectability assumption}~\ref{assump:detect} is in some sense a converse of the margin assumption, in that it ensures that there is sufficient $P_X$ measure near these threshold values. This assumption is necessary in order to ensure that the constraing in~\eqref{bayes_optimal_delta} is satisfied with high probability. 

% \todomgh{More discussion on DE and when it is really necessary, similar to the active paper.}

%--------------------------------------------------------- 
%--------------------------------------------------------- 
% BAYES OPTIMAL CLASSIFIER
%--------------------------------------------------------- 
%--------------------------------------------------------- 
\section{Bayes Optimal Abstaining Classifier}
\label{sec:bayes_optimal}
\vspace{-1em}
%We now describe the properties of an optimal deterministic classifier $g^*:\X \to \{0,1,\Delta\}$, or equivalently an optimal partition $g^*=(G^*_{-1},G^*_1,G^*_{\Delta})$, of the problem~(\ref{bayes_optimal_delta}). 
%We denote by $R\lp g\rp$, the classification risk $P_{XY}\lp g(X) \neq Y,\ g(X) \neq \Delta \rp$ associated 
%with the classifier $g$.%, and by $\eta(x) = P(Y=1|X=x)$, the regression function.
%\todo{Why not defining the classification risk in the problem setup section?}
%\todo{Added definition of $R(g)$ to Section~\ref{sec:problem_setup}}

%In the rest of this section, we proceed in two steps. We first consider the problem~\eqref{bayes_optimal_delta} under some strong regularity conditions on the input-label joint distribution, and obtain a result that suggests a threshold-type structure for the optimal classifier. Building upon this result, we then show that a threshold classifier is indeed Bayes optimal under much more general conditions as well. 

In this section, we derive the form of the Bayes optimal classifier for the problem \eqref{bayes_optimal_delta}, for an arbitrary input-label distribution $P_{XY}$. We begin by presenting a structural result about the optimal (deterministic) classifier, and build upon it to construct a randomized classifier, which is then shown to be Bayes optimal. Informally, an optimal abstaining classifier, in the fixed-cost as well as in the bounded-rate setting, must favor the abstain option in the regions of high ambiguity, or equivalently regions of low confidence. In the fixed-cost setting, this statement can be immediately obtained  by a pointwise comparison of the abstention cost $\lambda$ with the probability of misclassification, i.e.,~by a pointwise comparison of the three terms $\eta(x), 1-\eta(x)$, and $\lambda$. Our first result presents a way for formalizing this intuition in the bounded-rate setting. 

%-----------------------------------------------------------------------------
% BAYES OPTIMAL CLASSIFIER 1 
%-----------------------------------------------------------------------------
\begin{proposition}
\label{prop:bayes1}
Assume that the marginal $P_X$ has a density that satisfies~\ref{assump:mu_min} and~\ref{assump:mu_max}, and furthermore assume that the regression function $\eta(\cdot)$ is continuous. If $g^* = \lp G_{-1}^*, G_1^*, G_{\Delta}^*\rp$ is optimal among the deterministic abstaining classifiers that are feasible for \eqref{bayes_optimal_delta}, then for any $x_1 \in \text{int}(G_{-1}^*\cup G_1^*)$ and $x_2 \in \text{int} (G_{\Delta}^*)$, where ``$\text{int}$" refers to the interior, we must have \[ |\eta(x_1) - 1/2| \geq |\eta(x_2)-1/2|. \] 
\end{proposition}

The proof of this result proceeds by contradiction, and the details are given in Appendix~\ref{appendix:claim1}.
Proposition~\ref{prop:bayes1} motivates the following partition of the input space: $G_{-1}^* = \{x \in \X:\eta(x) < 1/2 - \gd \}$, $G_1^* = \{x \in \X :\eta(x) > 1/2 + \gd \}$, $\Gds = \{ x \in \X :|\eta(x)-1/2| < \gd \}$, $\partial G_1^* = \{x \in \X:\eta(x) = 1/2 + \gd \}$, and $\partial G_{-1}^*  = \{x \in \X:\eta(x) = 1/2 - \gd \}$, 
% \begin{align*}
%     G_{-1}^* &= \{x \in \X :\ \eta(x) < 1/2 - \gd \}, \\
%     G_1^* &= \{x \in \X :\ \eta(x) > 1/2 + \gd \}, \\
%     \Gds &= \{ x \in \X :\ |\eta(x)-1/2| < \gd \},\\
%     \partial G_1^* &= \{x \in \X:\ \eta(x) = 1/2 + \gd \}, \\
%     \partial G_{-1}^* & = \{x \in \X:\ \eta(x) = 1/2 - \gd \},
% \end{align*}
%
where $\gd$ is defined as 
\begin{equation}
    \label{eq:threshold-def}
\gd \coloneqq \sup \big\{ \gamma>0\ :\ P_X\lp |\eta(X)-1/2| \leq \gamma\rp \leq \delta \big\}. 
\end{equation}
Furthermore, let $\delta_1 = P_X\lp \Gds\rp \leq \delta$, $\delta_2 = \delta_1 + P_X\lp 
\partial G_1^* \cup \partial G_{-1}^* \rp \geq \delta$, and define $c_0 = \frac{\delta - \delta_1}{\delta_2 - \delta_1}$, where we use the convention $0/0 = 0$.

%Our next result shows that the deterministic threshold-type classifier is indeed optimal for~\eqref{bayes_optimal_delta} in all cases, with the exception of situations where the cumulative distribution function of $|\eta(X)-1/2|$ has jumps at $\gd$. In such cases, randomization would be necessary. 
Our next result tells us that if $\delta_1 = \delta$ then a deterministic classifier is Bayes optimal, while for arbitrary joint distributions $P_{XY}$, randomization is required. 

%-----------------------------------------------------------------------------
% BAYES OPTIMAL CLASSIFIER 2
%-----------------------------------------------------------------------------

%
\begin{theorem}
    \label{theorem:bayes_optimal}
%For a given $\delta>0$, if the term $\gd$ defined by~\eqref{eq:threshold-def} is such that $P_X\lp |\eta - 1/2| \leq \gd \rp = \delta$, then the classifier $g^* = \lp G_{-1}^*, G_1^*, G_{\Delta}^* \rp$ is optimal for the problem~\eqref{bayes_optimal_delta}. 

% If $P_X\lp \Gds \rp = \delta$, then the Bayes  optimal classifier for~\eqref{bayes_optimal_delta}, denoted by $g^*$ is deterministic and is given by $g^* = (G_{-1}^*, G_1^*, \Gds)$. In other cases, the following randomized classifier achieves the minimum risk:
%
For any arbitrary joint distribution $P_{XY}$, the following randomized classifier achieves the Bayes optimal risk for the problem~\eqref{bayes_optimal_delta}:
\begin{equation}
\label{eq:stochastic-optimal}
    g^* = \lp g^*_{-1}, g^*_{1}, \gD^* \rp \coloneqq 
    \begin{cases}
    (0,1,0) & \text{for } x \in G_1^*, \\
    (1,0,0) & \text{for } x \in G_{-1}^*, \\
    (0,0,1) & \text{for } x \in \Gds, \\
    (0,c_0, 1-c_0) & \text{for } x \in \partial G_1^*, \\
    (c_0, 0, 1-c_0) & \text{for } x \in \partial G_{-1}^*.
    \end{cases}
\end{equation}
Furthermore, in the special case when $P_X\lp \partial G_1 \cup \partial G_{-1}\rp = 0$ (i.e., $P_X\lp \Gds \rp = \delta)$, the optimal classifier reduces to the deterministic classifier $g^* = \lp G_{-1}, G_1, \Gds\rp$. 
\end{theorem}

% \todomgh{Definition of randomized classifier is missing. Why Prop~1 is needed?}

%-----------------------------------------------------------------------------
The proof of this statement is given in Appendix~\ref{appendix:claim2}. We have also included a separate simpler proof for the deterministic case  as we will employ similar arguments in later proofs. 
%-----------------------------------------------------------------------------
%-----------------------------------------------------------------------------
%Remark on the Bayes optimal for fixed cost setting. 
\begin{remark}
\label{remark:chow_comparison}
We note that the condition $P_X(\Gds)=\delta$ is satisfied, if the cdf of $|\eta(X)-1/2|$ is continuous. The Bayes optimal classifier under this condition was first obtained by~\cite{chow1957optimum}, and was also re-derived by \cite{denis2015consistency}.  We remove this technical assumption, thus, obtaining a characterization of Bayes optimal classifiers for arbitrary $P_{XY}$, while also providing an alternative proof for the continuous case. 
\end{remark}

\begin{remark}
\label{remark:fixed_cost}
For simplicity in the rest of the paper, we will restrict our attention to the case where the cdf of $|\eta - 1/2|$ has no jump at $\gd$, in which the Bayes optimal classifier is deterministic. 
% This implies that $P_X(\Gds) = \delta$, and thus, the Bayes optimal classifier is a deterministic threshold type classifier as noted in~\citet{chow1957optimum}.
This optimal classifier coincides with the optimal classifier for the classification problem with a fixed cost of abstaining (e.g.,~\citealt{Chow1970:Optimum,herbei2006classification,bartlett2008classification,cortes2016learning}).
However, the key difference is that unlike the fixed cost setting, the threshold $\gd$ is not known to the learner and must be estimated from the training data, thus, adding an additional layer of complexity to the problem.%\todo{add the no jump case.}
\end{remark}

%--------------------------------------------------------- 
%--------------------------------------------------------- 
% CONSTRUCTION OF CLASSIFIERS
%--------------------------------------------------------- 
%--------------------------------------------------------- 
\section{Plug-in Classifier with Randomization}
\label{sec:plugin_classifier}
\vspace{-1em}

In this section, we present a simple plug-in classifier whose construction consists of two steps: (i) construct an estimator of the regression function $\eta(\cdot)$ using $n$ labelled training samples, and (ii) determine the region of the input space to abstain using $m$ unlabelled samples.  

\noindent{\textbf{Step 1: Estimating the regression function.}} Before describing the details of the estimator, we need to introduce some more notation. For any $0<h<1$, we partition the input space $\X=[0,1]^D$ into $M_h\coloneqq\lceil\frac{1}{h}\rceil^D$ cubes (cells) denoted by $E_{h,1}, \ldots, E_{h,M_h}$. Let $\mathcal{E}_h=\{E_{h,1}, \ldots, E_{h,M_h}\}$ be the partition (the set of these $M_h$ cubes). Each input point $x\in\X$ belongs to a single cube $E_{h,i}$ in the partition $\mathcal{E}_h$. Mapping $i_h:\X \mapsto [M_h]$, where $[M_h]=\{1,\ldots,M_h\}$, takes a point $x\in\X$ as input and returns the index of the cube it belongs to, i.e.,~if $x\in E_{h,i}$, then $i_h(x)=i$.

%------------------------
% Piecewise constant estimator for a fixed $h$
%------------------------

For a given partition $\mathcal{E}_h$ and $n$ labelled training samples $S_l = \{(X_j,Y_j)\}_{j=1}^n$, we define the piecewise constant estimator of the regression function as
\begin{equation*}
\hat{\eta}_h(x)=
\begin{cases}
\frac{\sum_{j:i_h(X_j)=i_h(x)}Y_j}{n(E_{h,i_h(x)})}, & \text{if } n(E_{h,i_h(x)})>0, \\
\frac{1}{n}\sum_{j=1}^nY_j, & \text{otherwise},
\end{cases}
\end{equation*}
where $n(E_{h,i})$ is the number of the training samples in the cube $E_{h,i}$. %For any $x \in \X$, we define the index $i_h(x) \coloneqq \{i\; \mid \; \pi_h(x) = x_{h,i}\}$. 
Then, the estimation error at any point $x$ for the classifier $\hat{\eta}_h(\cdot)$ can be written as
\begin{align}
\label{eq:estimation-error}
    |\eta(x)-\hat{\eta}_h(x)| \leq & \; |\hat{\eta}_h(x) - \bar{\eta}\lp E_{h,i_h(x)}\rp |  + | \bar{\eta}\lp E_{h,i_h(x)}\rp - \eta(x)|,
\end{align}
where $\bar{\eta}(E_{h,i})=\lp1/P_X(E_{h,i})\rp\int_{E_{h,i}}\eta(x)dP_X(x)$ is the average $\eta(\cdot)$ value in the cube $E_{h,i}$. The first error term on the RHS of~\eqref{eq:estimation-error} is due to the observation noise and the second error term is due to the variation of the regression function values in the cell $E_{h,i_h(x)}$. These two error terms have opposite dependence on the parameter $h$; as $h$ increases the first (stochastic) term reduces due to the smoothing effect of larger grid size, while the second (deterministic) term increases. As we will see in Proposition~\ref{prop:concentration}, we define an upper-bound $e_S(h,x)$ for the stochastic term that is roughly proportional to $(nh^D)^{-1/2}$ and an upper-bound $e_D(h,x)$ for the deterministic term that is proportional to $h^{\beta}$, assuming that $\eta(\cdot)$ is H\"older continuous with parameters $(L,\beta)$. Thus, the optimal choice of $h$ (up to a factor of 2) is $\tilde{\mathcal{O}}\lp n^{-1/(2\beta + D)}\rp$, which balances the two terms. 

The optimal choice of the parameter $h$ requires the knowledge of the parameters $L$ and $\beta$ that may not be known to the learner. We now describe a data driven approach for selecting the appropriate grid size $h$. Our approach employs a modification of the Lepski's estimator selection procedure~\citep[\S~3.2]{nemirovski2000topics} to choose the best grid size $h$, which allows us to obtain pointwise control over the estimation error. 

%--------------------------------------------------------------------------
%Concentration results
%--------------------------------------------------------------------------
%--------------------------------------------------------------------------
% GRID SIZE SELECTION RULE
%--------------------------------------------------------------------------
Based on the concentration inequalities given in Proposition~\ref{prop:concentration} in Appendix~\ref{appendix:concentration}, we can obtain an upper-bound (with high probability) of the form $\sqrt{\frac{8\log(2n/h^D)}{n\mu_{\min}h^D}}$ on the first term in~\eqref{eq:estimation-error}. Since we will restrict our attention to $h \geq  1/N$, we can further upper-bound this term and define $e_S(h,x)=\sqrt{\frac{32\log(n \mu_{\min})}{n\mu_{\min}h^D}}$. The second term in~\eqref{eq:estimation-error} is the difference between $\eta(x)$ and the average $\eta(\cdot)$ value in the cell $E_{h,i}$. We upper-bound this term by the maximum variation of $\eta(\cdot)$ in the cell $E_{h,i}$ and define $e_D(h,x)=\sup_{z_1,z_2\in E_{h,i_h(x)}}|\eta(z_1)-\eta(z_2)|$. In the case where $\eta(\cdot)$ is assumed to be H\"older continuous with parameters $(L,\beta)$, we may define $e_D(h,x)=L(\sqrt{D}h)^\beta$.

We can now  define the estimator $\hat{\eta}(\cdot)$ as 
\begin{equation}
\label{eq:estimtor-def1}
\forall x\in\X, \qquad \hat{\eta}(x) = \hat{\eta}_{\hat{h}_x}(x),    
\end{equation}
with $\hat{h}_x$ selected according to the rule
\begin{align}
\label{eq:estimtor-def2}
    \hat{h}_x  \coloneqq \max \big\{ h & \in H\;\mid\; |\hat{\eta}_h(x)-\hat{\eta}_{h'
    }(x)| \leq 4e_S(h',x),
 \forall \; h' \in H,\ h' \leq h\big\}.
\end{align}
We now state a pointwise bound on the error of the estimator $\hat{\eta}(\cdot)$ defined by~\eqref{eq:estimtor-def1} and~\eqref{eq:estimtor-def2}.

\begin{proposition}
\label{prop:regression_estimator}
Suppose the events $\Omega_1$ and $\Omega_2$ introduced in Proposition~\ref{prop:concentration} hold. Then if the number of labeled training samples $n$ is large enough to ensure that $Nh^*_x\geq 2D$, for all $x \in \X$, where $D$ is the dimension of the input space $\X$, $N$ defined as in the statement of Proposition~\ref{prop:concentration}, and $h^*_x \coloneqq \max \{ h\in (0,1) \mid e_S(h,x) \geq e_D(h,x) \}$, we have 
\[
\forall x\in\X, \quad\;\;\; |\hat{\eta}(x) - \eta(x)| \leq 9e_S(h_x^*,x).
\] 
Furthermore, if~\ref{assump:holder} holds, then we have
$|\hat{\eta}(x)-\eta(x)|\leq b_n = \tilde{\mathcal{O}}(n^{-\beta/(2\beta+D)})$, for all $x \in \X$.
\end{proposition}

\begin{proofoutline}
The proof is given in Appendix~\ref{appendix:regression_estimator}.
\end{proofoutline}

\begin{remark}
\label{remark:h_star_assumption}
The assumption $N h^*_x  \geq 2D, \ \forall x \in \X$, essentially imposes the condition that the regression function $\eta(\cdot)$ does not change very sharply in any region of the input space. More formally, it assumes that $n$ is large enough to ensure that the variation of $\eta(\cdot)$ in any cell $E_{h,i}$ of size $h = k/N$, with $k \in \{1,2,\dots, 2D-1\}$, denoted by $e_D(h,x)$, is smaller than $e_S(h,x)$.
\end{remark}

%------------------------------------------------------------------------------
%ESTIMATE THE ABSTAINING REGION
%------------------------------------------------------------------------------

\paragraph{Step~2: Estimate the abstaining region.} The second step in the construction of the plug-in classifier is to define the abstaining region using the estimator $\hat{\eta}(\cdot)$ defined by~\eqref{eq:estimtor-def1} and~\eqref{eq:estimtor-def2}. Since the true marginal $P_X$ is unknown and the measure of the abstaining region must be empirically estimated from the $m$ unlabelled samples, it is necessary to introduce some slack to ensure that the classifier is feasible for the problem~\eqref{bayes_optimal_delta}. Our next result presents an appropriate value of the slack. 

\begin{proposition}
\label{prop:slack}
Given $m$ unlabelled samples $S_u = \{X_j\}_{j=n+1}^{n+m}$, we define the empirical measure of a set $E$ as $\hat{P}_m(E)\coloneqq\frac{1}{m}\sum_{j=n+1}^{n+m}\indi_{\{X_j\in E\}}$. Then, the event $\Omega_3$ defined below occurs with probability at least $1-1/m$.
\begin{align*}
\label{eq:conc1}
\Omega_3 \coloneqq \Big\{\sup_{c>0}\big\{\big|\hat{P}_m\big(|&\hat{\eta}(x) - 1/2| \leq c\big) 
- P_X\big(|\hat{\eta}(x) - 1/2| \leq c\big)\big|\big\} \leq a_m \Big\},
\end{align*}
where the slack term $a_m$ is defined as
$a_m \coloneqq \sqrt{72\log(4m)/{m} }.$
\end{proposition}
\begin{proofoutline}
The result follows by using the VC inequality along with the fact that the VC dimension of the class of functions $\{ \indi_{\{|\cdot|\leq c\}} \mid c\in \mbb{R}\}$ is 2~\citep[\S~6.3.2]{shalev2014understanding}. The details of the proof are given in Appendix~\ref{appendix:slack}. 
\end{proofoutline}
Using the above results, we define the empirical estimate of the threshold as     
\begin{equation}
        \label{eq:threshold}
        \hat{\gamma} \coloneqq \sup \big\{ \gamma > 0 : \hat{P}_m(|\hat{\eta}(x) - 1/2| \leq \gamma) \leq \delta - a_m \big\}.
\end{equation}
Next we introduce the following sets: 
\begin{align*}
        & \hat{G}_{-1} = \{x \in \X :\ \hat{\eta}(x) < 1/2 -\hat{\gamma}-2b_n \}, \qquad \hat{G}_1 = \lc x \in \X :\ \hat{\eta}(x) > 1/2 + \hat{\gamma} + 2b_n \rc, \\
       & \partial\hat{G}_{-1} = \{x \in \X: \hat{\gamma}  <1/2 -\hat{\eta} \leq \hat{\gamma} + 2b_n\}, \; \partial \hat{G}_1 = \{x \in \X :\  \hat{\gamma} < \hat{\eta}-1/2 \leq \hat{\gamma} + 2b_n\} \\
    &\hat{G}_{\Delta} = \lc x \in \X :\ \left| \hat{\eta}(x) - 1/2 \right| \leq \hat{\gamma} \rc.
\end{align*}
Define $\hat{p}_1 = \hat{P}_m\lp \hat{G}_\Delta\rp$, $\hat{p}_2 = \hat{P}_m\lp \hat{G}_\Delta \cup \partial \hat{G}_{-1} \cup \partial \hat{G}_1 \rp$, and $\hat{c} \coloneqq (\delta - 5a_m)/(\hat{p}_2 - \hat{p}_1)$ if $\hat{p}_1 < \delta-5a_m$, else $\hat{c}=0$. 
Using the above terms, we can define a randomized classifier as $\hat{g}$ such that 
\begin{align}
\label{eq:plugin-classifier}
\hat{g}(x) & = i \;\text{ for }\; x \in \hat{G}_i,\; i \in \{-1,1,\Delta\}, \\
\hat{g}(x) & = (1-\hat{c}, 0, \hat{c}) \text{ for }\; x \in \partial\hat{G}_{-1}, \\
\hat{g}(x) & = (0, 1-\hat{c}, \hat{c}) \text{ for }\; x \in \partial \hat{G}_1, 
\end{align}

We now prove an upper-bound on the excess misclassification error of the plug-in classifier $\hat{g}$ defined by~\eqref{eq:threshold} and~\eqref{eq:plugin-classifier} (see Appendix~\ref{appendix:excess_risk_bound} for the proof).
\begin{theorem}
\label{theorem:excess_risk_bound}
Suppose assumptions~\ref{assump:margin},~\ref{assump:holder} and~\ref{assump:mu_min} hold, and the number of the labelled and unlabelled samples, $n$ and $m$, are large enough. Then, for the plug-in classifier $\hat{g}$, defined by~\eqref{eq:threshold} and~\eqref{eq:plugin-classifier}, the following statements are true with probability at least $1 - 1/m - 2/n$:
\begin{enumerate}
    \item $P_X\lp \hat{g}(X)=\Delta\rp  \leq \delta$.
    \item The excess probability of misclassifiction (excess risk) of the plug-in classifier $\hat{g}$ over the optimal classifier $g^*$ satisfies $R \lp \hat{g} \rp - R\lp g^* \rp \leq 5a_m + 4C_0(5b_n)^{1+\rho_0}$, where $a_m = \Oh(\sqrt{ \log m /m })$ and $b_n = \Oh(\sqrt{\log n}\;n^{-\frac{\beta}{2\beta + D}})$.    
\end{enumerate}
\end{theorem}

\begin{remark}
\label{remark:denis_hebiri_comparison}
\cite{denis2015consistency} proposed a  general plug-in scheme which takes in any consistent estimator of $\eta(\cdot)$ and constructs an abstaining classifier which asymptotically satisfies the constraint in~\eqref{bayes_optimal_delta}. 
Our approach differs from theirs in two important ways: \textbf{1)} \cite{denis2015consistency} construct the abstain region by taking the inverse of the empirical cdf of $|\hat{\eta}-1/2|$, which imposes continuity requirements on the empirical cdf, and thus, restricts the class of estimators of $\eta$ that can be used. For instance, the piecewise constant estimator that we have constructed above does not satisfy their assumption \textbf{A2}. 
On the other hand, we employ a  randomized strategy motivated by the form of the Bayes optimal in Theorem~\ref{theorem:bayes_optimal}, which imposes no continuity restrictions on the estimator of the regression function. 
\textbf{2)} In many problem instances, it is desirable that the bounded-rate constraint is strictly satisfied (see~\cite[\S~3.1]{rigollet2011neyman} for a similar discussion in context of Neyman-Pearson classification). Accordingly, our randomized approach implies that the abstention constraint is satisfied with high probability. This is in contrast to the classifier constructed by \cite{denis2015consistency}, which can only satisfy the constraint asymptotically. 
\end{remark}

\begin{remark}
\label{remark:adaptivity}
An important feature of our proposed classifier is that it automatically adapts to the local smoothness of the regression function. While this data-driven adaptivity to the smoothness parameters comes at the cost of an additional $\log n$ factor in $b_n$, it can result in much faster convergence rates in spatially inhomogeneous functions. 
More specifically, if $\eta$ is steep near the boundaries $(\beta \approx 1)$ and flat away from it $(\beta \approx 0)$, then for $n$ large enough, the convergence rates of our algorithm will only depend on the local smoothness near the boundaries.
\end{remark}

\noindent{\textbf{Lower Bound.}} We conclude this section by deriving a minimax lower-bound on the excess risk for the class of problems considered, i.e., $P_{XY}$ satisfying the assumptions ~\ref{assump:margin},~\ref{assump:holder}.  This lower bound demonstrates  the near-optimality of our adaptive plug-in classifier. To the best of our knowledge, this is the first lower-bound result for the problem of classification with abstention. 
% \todomgh{We may wanna mention the 1st lower-bound result both in the intro and abstract.}

\begin{theorem}
\label{theorem:lower_bound}
Let $\mathcal{A}$ represent any algorithm that learns an abstaining classifier $\hat{g}$ and let $\mathfrak{P}\lp \beta, \rho_0\rp$ represent the class of $P_{XY}$ satisfying assumptions~\ref{assump:margin} and~\ref{assump:holder}. Then, we have 
\begin{equation*}
    \inf_{\mathcal{A}} \sup_{P_{XY} \in \mathfrak{P}(\beta, \rho_0)} R(\hat{g}) - R(g^*) = \Omega\big(n^{-\beta(1+\rho_0)/(2D + \beta)}\big). 
\end{equation*}
\end{theorem}
\begin{proofoutline}
The proof follows the general outline described in~\citet{audibert2007fast} with two modifications: \textbf{1)} a new comparison inequality and \textbf{2)} construction of a new class of \emph{hard} problem instances. The details of these two steps can be found in the proof of Theorem~2 in \citep{anonymous2019active}.
\end{proofoutline}

%--------------------------------------------------------- 
%--------------------------------------------------------- 
% COMPUTATIONALLY FEASIBLE ALGORITHMS
%--------------------------------------------------------- 
%--------------------------------------------------------- 
\section{Computationally Feasible Algorithms}
\label{sec:convex}
\vspace{-1em}
 The implementation of the plug-in classifier of Section~\ref{sec:plugin_classifier} requires an exhaustive search over a uniform grid partitioning the input space; an operation with an exponential runtime complexity.
 We now present two computationally tractable algorithms (in Sections~\ref{subsec:binary_search} and~\ref{subsec:algorithm1}) for constructing abstaining classifiers, as their implementation involves solving convex programs. 

Following~\citet{cortes2016learning} and~\citet{elyaniv2010selective}, we now consider classifiers $g$ represented by the pair $(h,r)$, where the sign of $h$ is used for predicting labels and the sign of $r$ decides whether to abstain or not. The original problem~\eqref{bayes_optimal_delta} can now be re-written as 
\begin{equation}
\tag{$CA_{\delta,2}$}
\label{abstain2}
\begin{aligned}
 \underset{h,r}{\text{$\min$}} \quad \mathbb{E}\Big[l\big(h(X)Y,r(X)\big)\Big], \qquad \text{subject to} \quad  \mathbb{E} \lb \indi_{\{r\lp X \rp \leq 0 \}}\rb  \leq \delta,
\end{aligned}
\end{equation}
where the loss function $l$ is defined as $l(z_1,z_2) \coloneqq \indi_{\{z_1 \leq 0\}}\indi_{\{-z_2<0\}}$.

%--------------------------------------------------------- 
% BINARY SEARCH WITH COST-BASED REJECTION
%--------------------------------------------------------- 

\subsection{Binary Search with Cost-based Rejection}
\label{subsec:binary_search}

As noted in Remark~\ref{remark:fixed_cost}, the optimal solution to the problem~\eqref{bayes_optimal_delta} is the same as that in the  fixed-cost setting, with cost  equal to $1/2-\gd$.  Classification  with a fixed cost of abstention $\lambda \in (0,1/2)$ involves minimizing the cost function 
%\todo{We need a reference here} 
%
$l_{\lambda}(g,x,y) \coloneqq \indi_{\{g(x)\neq y\}} \indi_{\{g(x)\neq \Delta\}} + \lambda \indi_{\{g(x) = \Delta\}}$.
Several works in the literature, such as~\citet{bartlett2008classification, yuan2010classification,cortes2016learning}, have proposed convex surrogates $\varphi_\lambda(\cdot)$ to this loss function that are calibrated w.r.t.~the Bayes optimal solution. More specifically, computationally tractable algorithms minimize the cost $\mbb{E}\big[ \varphi_\lambda(h,r,X,Y)\big]$ for $(h,r) \in \mathcal{H}\times \mathcal{R}$ for suitable choices of $\mathcal{H}$ and $\mathcal{R}$. 
%
% \begin{equation}
% \label{eq:fixed_cost_convex}
%     \min_{(h,r) \in \mathcal{H} \times \mathcal{R}} \mbb{E}\big[ \varphi_\lambda(h,r,X,Y)\big],
% \end{equation}
% %
% for some appropriate choice of function classes $\mathcal{H}$ and $\mathcal{R}$. 
The definitions of these convex surrogate functions rely heavily on the knowledge of the abstention cost $\lambda$, which  cannot be determined beforehand for the bounded rate setting. Thus the existing approaches to defining convex surrogate loss functions are not applicable here. We now propose a computationally feasible algorithm for~\eqref{bayes_optimal_delta} that leverages the above mentioned connection between the optimal solutions of~\eqref{bayes_optimal_delta} and the problem of classification with fixed-cost of abstention. 

\noindent{\textbf{Algorithm~1:}}
Take as input the labeled and unlabeled training sets $S_l$ and $S_u$, a slack term $\alpha_m$, an interval $\mathcal{I}_n \subset (0,\delta)$, and an algorithm $\mathcal{A}$ for learning with fixed-cost of abstention. Set $L_1 = 0$ and $U_1 = 1/2$. For $k \in \{1,2,\dots,\}$ perform the following steps:
\begin{enumerate}
    \item Set the rejection cost to $\lambda_k = (L_k + U_k)/2$. 
    \item Use the algorithm $\mathcal{A}$ and learn a classifier $\hat{g}_{\lambda_k} = (h_k,r_k)$ with abstention cost $\lambda_k$ on the labelled training set $S_l$ with $n$ samples. 
    \item Compute $Q_k \coloneqq \frac{1}{m}\sum_{j=n+1}^{n+m}\indi_{\{r_k(X_j) \leq 0\}} + \alpha_{m}$ using the unlabelled training set $S_u$ with $m$ samples. 
    \item If $Q_k \in \mathcal{I}_{n}$, then stop, else if $Q_k \leq \delta$, update $U_k \leftarrow \lambda_k$, else update $L_k \leftarrow \lambda_k$. 
\end{enumerate}
In the fixed-cost setting, in addition to the calibration results, there exist proven bounds on the true excess risk in terms of the surrogate excess risk. More formally, if $\bar{R}_\lambda(g) = \mbb{E}[l_{\lambda}(g,X,Y)]$ and $\bar{R}_{\varphi_{\lambda}}(g)=\mbb{E}[\varphi_{\lambda}(g,X,Y)]$ denote the risk and the convex risk, respectively, then we have $\bar{R}_\lambda(g) - \bar{R}_\lambda(g_\lambda) \leq \Psi\lp \bar{R}_{\varphi_{\lambda}}(g)-\bar{R}_{\varphi_{\lambda}}(g_\lambda)\rp$, where $g_\lambda = \argmin_{g}\bar{R}_{\lambda}(g)$ represents the optimal abstaining classifier for the fixed-cost setting and $\Psi:[0,\infty)\mapsto [0,\infty)$ is some non-decreasing function with $\Psi(0)=0$. Our next result exploits this property to obtain bounds on the excess risk of the classifier returned by our proposed Algorithm~1.
% \todo{Let's give these algorithms a label to be able to refer to them easily.} 
%
\begin{theorem}
\label{theorem:binary_search}
Suppose the following conditions are satisfied:
\vspace{-0.5em}
\begin{enumerate}
\item Assumptions~\ref{assump:margin},~\ref{assump:detect}, and~\ref{assump:mu_min} hold. 
\item  The Bayes optimal classifier lies in the function class $\mathcal{H}\times \mathcal{R}$, for all $\lambda \in (0,1/2)$.
\item  The convex cost function $\varphi_\lambda$ is calibrated for all values of $\lambda \in (0,1/2)$.
\item  $\bar{R}_{\varphi_{\lambda}}\lp \hat{g}_{\lambda}\rp - \bar{R}_{\varphi_\lambda}\lp g_\lambda\rp \leq \bar{A}_n$ w.p. at least $1-1/n$, where $\lim_{n \to \infty}\bar{A}_n = 0$. Moreover, define $A_n = \Psi(A_n)$, where $\Psi(\cdot)$ is a non-decreasing function with 
$\Psi(0) = 0$. 
\end{enumerate}
Suppose $\hat{g}_{\lambda}$ is the output classifier when Algorithm~1 is run over the function class $\mathcal{H}\times \mathcal{R}$ with parameters $\alpha_m=2\mathfrak{R}_m(\mathcal{R}) + \sqrt{2\log(2m)/m}$, where $\mathfrak{R}_m$ is the Rademacher complexity of $\mathcal{R}$, and $\mathcal{I}_n = [\delta-3{B}_n, \delta-2{B}_n]$, where $B_n \coloneqq 4C_0\big(\frac{A_n}{C_0}\big)^{\rho_0/(\rho_0+1)}$. Then, for $m$ and $n$ large enough, with probability at least $1-1/n - 1/m$, we have
\begin{equation*}
\label{eq:binary_search_risk}
        R(\hat{g}_\lambda) - R(g^*) \leq A_n + 4(1/2 - \gd) K_{m,n} + 4\big(K_{m,n}^{1+1/\rho_1}/C_1^{1/\rho_1}\big),
\end{equation*}
where $K_{m,n} \coloneqq (5B_n + 2\alpha_m)/2$.
\end{theorem}
The proof of this result is given in Appendix~\ref{appendix:binary_search}. An concrete example of the terms $\bar{A}_n$ and $\Psi(\cdot)$ is given in Remark~\ref{remark:Psi} in Appendix~\ref{appendix:binary_search}.

% \begin{remark}
% \label{remark:Psi}
% A concrete example of the terms $\bar{A}_n$ and $\Psi(\cdot)$  can be obtained from Corollary~19 in~\citet{yuan2010classification}. Here $\mathcal{H}$ is some class of functions $h:\X \mapsto \mbb{R}$ and 
% $\mathcal{R} = \{\indi_{\{|h|>c\}}\ \mid \ h \in \mathcal{H}, \ c \in [0,\infty)\}$. If $N_n$ denotes the $1/n$ covering number of $\mathcal{H}$ w.r.t.~the uniform metric and $\varphi_\lambda(\cdot)$ is a convex surrogate satisfying the conditions of Theorem~9 in~\citet{yuan2010classification}, then we have $\bar{A}_n = \mathcal{O}\lp \frac{1}{n} + \frac{\log(n N_n)}{n}\rp$ and $\Psi(x) = x^{1/(s + \beta - s\beta)}$, for some $s>0$ and $\beta= \rho_0/(1+\rho_0)$.
% \end{remark}

%--------------------------------------------------------- 
% CONVEX SURROGATE WITH CONVEX CONSTRAINTS
%--------------------------------------------------------- 

\subsection{Baseline: Convex Surrogate with Convex Constraints}
\label{subsec:algorithm1}
\vspace{-0.8em}
Since the loss function $l$ is not convex, we now propose a baseline algorithm which  employs a convex surrogate as in~\citet{cortes2016learning} using a convex function $\varphi_1(\cdot)$, which is an upper-bound on $\indi_{\{\cdot \leq 0\}}$ as
$
    l(z_1,z_2) = \indi_{\{z_1 \leq 0\}}\indi_{\{-z_2 <0\}} \leq \indi_{\{\max\{z_1, -z_2\}\leq 0\}} \leq \indi_{\{ \frac{z_1-z_2}{2} \leq 0 \} } \leq \varphi_1\big( \frac{z_1 -z_2}{2}\big).
$
Similarly, we can also replace the constraint with its convex relaxation by employing another function $\varphi_2(\cdot)$ to upper bound the indicator $\indi_{\{\cdot \leq 0\}}$. Note that since $\varphi_2(\cdot)$ is an upper bound on $\indi_{\{\cdot \leq 0\}}$, we are restricting the set of feasible solutions. 

% Similarly, we can use $\varphi_2(\cdot)$, another convex upper-bound on $\indi_{\{\cdot \leq 0\}}$, for the constraint to obtain the following convex relaxation of the problem~\eqref{abstain2}:
% %
% \begin{equation}
%     \tag{$CS_1$}
%     \label{convex_surrogate_1}
%     \begin{aligned}
%      \underset{h,r}{\min} \quad \mathbb{E}\lb \varphi_1\lp \frac{ h(X)Y - r(X) }{2} \rp \rb, \qquad \text{subject to} \quad \mathbb{E}\Big[\varphi_2\big((r(X)\big)\Big] \leq \delta.
%     \end{aligned}
% \end{equation}
% %
% Note that since $\varphi_2(\cdot)$ is an upper-bound on $\indi_{\{\cdot \leq 0\}}$,~\eqref{convex_surrogate_1} has a smaller feasible set than~\eqref{abstain2}. 

{\textbf{Algorithm~2:}} 
We now describe an algorithm that solves the following empirical version of the convex relaxation of~\eqref{abstain2}, in which both $\varphi_1$ and $\varphi_2$ are set to the hinge loss $\varphi_H(z) \coloneqq \max \{0, 1-z\}$:
\begin{equation}
    \tag{$CS_2$}
    \label{convex_surrogate_2}
    \begin{aligned}
     \underset{(h,r) \in \mathcal{H} \times \mathcal{R}}{\min} \; \frac{1}{n}\sum_{j=1}^n\varphi_H\lp \frac{h(X_j)Y_j - r(X_j)}{2}\rp, \quad\; \text{s.t.} \;\;\; \frac{1}{m}\sum_{j=n+1}^{n+m}\varphi_H\big( r(X_j)\big) \leq \delta - \frac{\tau}{\sqrt{m}},
    \end{aligned}
\end{equation}
where $\mathcal{H}$ and $\mathcal{R}$ are some function classes. The slack term $\tau/\sqrt{m}$ is introduced in the empirical constraint of~\eqref{convex_surrogate_2} in order to ensure that the feasible functions for~\eqref{convex_surrogate_2} also satisfy the constraint of~\eqref{abstain2} with high probability. An appropriate choice of $\tau$ depending on the function class $\mathcal{R}$ is given in Proposition~\ref{prop:empirical_constraint} in Appendix~\ref{appendix:binary_search}.

%--------------------------------------------------------- 
%--------------------------------------------------------- 
% EXPERIMENTS
%--------------------------------------------------------- 
%--------------------------------------------------------- 
\section{Experiments}
\label{sec:experiments}
\vspace{-1em}
\begin{wrapfigure}{r}{0.5\textwidth}
\label{fig:pima}
  \begin{center}
    \includegraphics[width=0.48\textwidth]{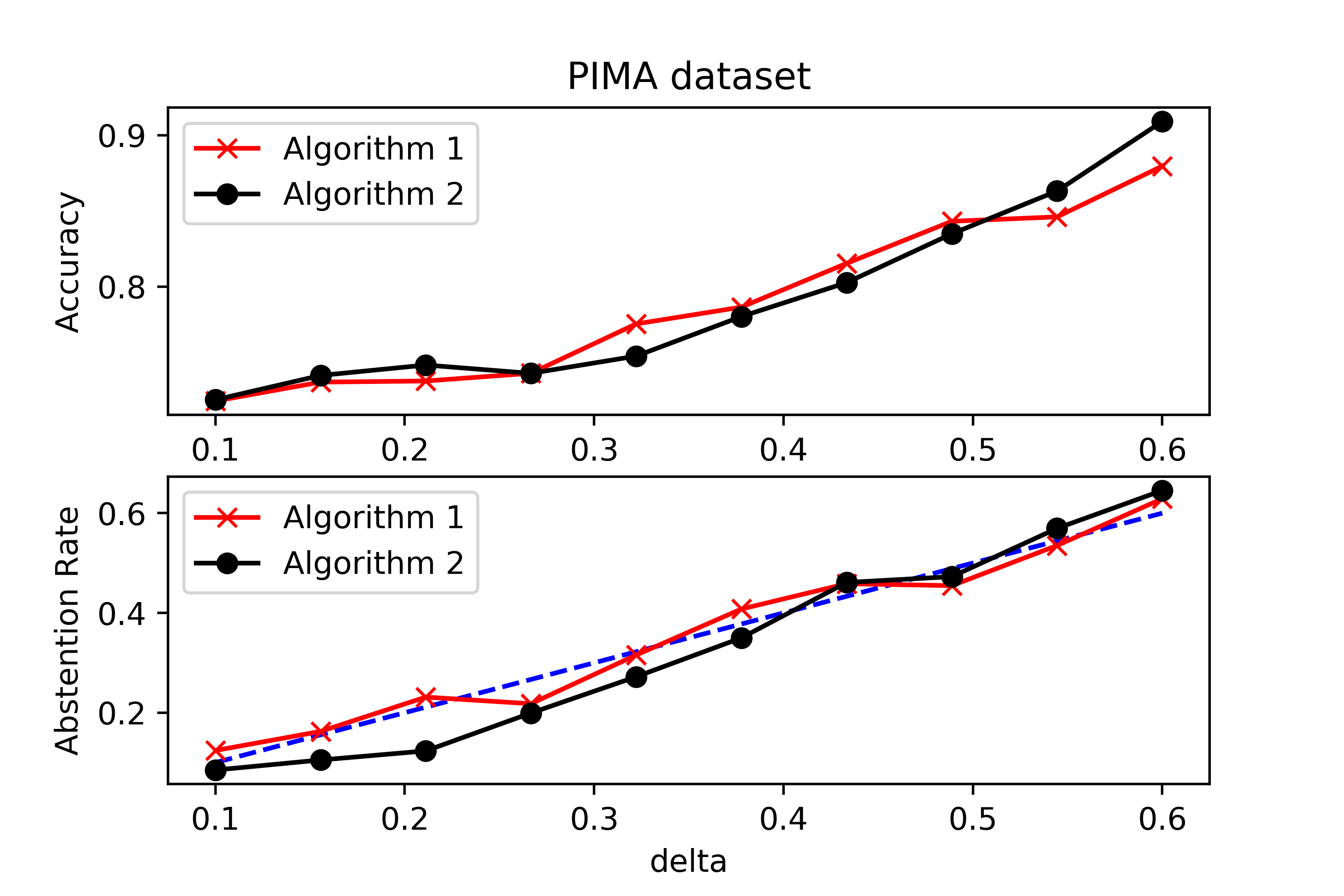}
  \end{center}
 \caption{Plot of the rejection rate versus accuracy as $\delta$ varies from $0.1$
to $0.6$ for the two algorithms on the PIMA dataset.}
\end{wrapfigure}

We now describe some empirical results on the performance of our proposed algorithms. We  emphasize that the goal of these experiments is not to construct the most accurate  classifiers, but to demonstrate that \emph{(i)} abstention improves classification accuracy, and \emph{(ii)} our proposed algorithms can achieve tight control over the abstention rate. 
We implemented the algorithms using CVXPY \citep{cvxpy}.
Figure~\ref{fig:pima}  shows the plot of rejection rate versus classification
accuracy for the PIMA dataset, as the parameter $\delta$ was varied from $0.1$ to $0.6$. As expected, 
the classification accuracy increases with increase in $\delta$. 
Algorithm~1 was able to find classifiers with very rejection rates very close 
to $\delta$, while Algorithm~2 learned classifiers which were more conservative since it searches over a smaller set.
Additional experiments are reported in Appendix~\ref{appendix:details_experiment}. 

% \begin{figure}[h]
% \label{fig:digits}
% \centering
% \includegraphics[ width=0.5\linewidth]{UTILS/fig_pima.png}
% \caption{Plot of the rejection rate versus accuracy as $\delta$ varies from $0.1$
% to $0.6$ for the two algorithms on the digits dataset.}
% \end{figure}

%-------------------------------------------------------------------
%-------------------------------------------------------------------

%-------------------------------------------------------------------
%-------------------------------------------------------------------
%-------------------------------------------------------------------
%-------------------------------------------------------------------

\newpage
\bibliographystyle{apalike}
\bibliography{ref}

%-------------------------------------------------------------------
%-------------------------------------------------------------------
% APPENDIX A
%-------------------------------------------------------------------
%-------------------------------------------------------------------

\newpage 
\onecolumn 
\begin{appendix}
\section{Deferred Proofs from Section~\ref{sec:bayes_optimal}}

%-------------------------------------------------------------------
% APPENDIX A.1
%-------------------------------------------------------------------

\subsection{Proof of Proposition~\ref{prop:bayes1}}
\label{appendix:claim1}

\begin{proofoutline}
The proof proceeds by contradiction. We assume that there exists an optimal abstaining classifier $g^* = \lp G_{-1}^*, G_1^*, \Gds \rp$ with points $x_1 \in \text{int}\lp G_{-1}^*\cup G_1^*\rp$ and $x_2 \in \text{int}(\Gds)$, such that $|\eta(x_1)-1/2| < |\eta(x_2)-1/2|$. Then using the continuity of $\eta$ and the assumptions on $P_X$, we can find appropriate open balls in $\lp G_{-1}^* \cup G_1^*\rp$ and $\Gds$ with the same $P_X$ measure, and use them to construct a new classifier satisfying the constraint of $\eqref{bayes_optimal_delta}$ with strictly better objective function value. This contradicts the optimality of $g^*$.
\end{proofoutline}

\begin{proof}
We proceed by contradiction. Suppose there exist $x_1 \in \text{int} \lp G_{-1}^*\cup G_1^*\rp$ and $x_2 \in \text{int}(G_{\Delta}^*)$ such that we have $|\eta(x_1)-1/2| < |\eta(x_2)-1/2|$. We will show that if this happens, then we can construct a classifier that satisfies the constraint in (\ref{bayes_optimal_delta}) and has a strictly smaller risk $R(\cdot)$, and thus, contradicting the optimality of $g^* = (G_{-1}^*,G_1^*, G_{\Delta}^*)$. 

We proceed in the following steps:

\begin{itemize}
\item Since $x_1 \in \text{int}(G_{-1}^*\cup G_1^*)$ and $x_2 \in \text{int}(G_{\Delta}^*)$, we can select an $\epsilon_1>0$, such that $B(x_1,\epsilon_1) \subset \text{int}(G_{-1}^*\cup G_1^*)$ and $B(x_2,\epsilon_1) \subset \text{int}(G_{\Delta}^*)$, where we denote by $B(x,\epsilon)$, the ball centered at $x$ with radius $\epsilon$. 

\item We define $\alpha_0 \coloneqq |\eta(x_2)-1/2| - |\eta(x_1)-1/2| >0$. By the continuity of $\eta $, there exists an $0<\epsilon_2\leq\epsilon_1$ such that $\sup_{x \in B(x_2, \epsilon_2)}|\eta(x)-\eta(x_2)| \leq \alpha_0/3$ and $\sup_{x \in B(x_1, \epsilon_2)}|\eta(x)-\eta(x_1)|\leq \alpha_0/3$. Thus, we may write  
\begin{align*}
\inf_{x \in B(x_2,\epsilon_2)} |\eta(x)-1/2| &\geq |\eta(x_2) - 1/2| - \sup_{x \in B(x_2, \epsilon_2)} |\eta(x) - \eta(x_2)|\\
& \geq |\eta(x_2) - 1/2| - \alpha_0/3,
\end{align*}
and similarly 
\[
\sup_{x \in B(x_1,\epsilon_2)} |\eta(x)-1/2| \leq |\eta(x_1)-1/2| + \alpha_0/3. \]
Together, these two inequalities imply that 
\begin{equation}
\label{eq:inf_minus_sup}
\inf_{x \in B(x_2, \epsilon_2)} |\eta(x) - 1/2| > \sup_{x \in B(x_1, \epsilon_2)}|\eta(x) - 1/2|.
\end{equation}
\item Assume that $P_X(B(x_1,\epsilon_2))> P_X(B(x_2, \epsilon_2))$ (the case in which $P_X(B(x_1,\epsilon_2))<P_X(B(x_2,\epsilon_2))$ can be handled similarly, while in the case of equality, we can skip this step). Since $P_X$ has a density w.r.t.~the Lebesgue measure which takes values in the range $[\mu_{\min}, \mu_{\max}]$, for any $\epsilon>0$, we
have $C_D\mu_{\min}\epsilon^D \leq P_X\lp B(x_1,\epsilon)\rp \leq C_D\mu_{\max}
\epsilon^D$, where $C_D$ is a constant depending on the dimension D.
This implies that the mapping $\epsilon
\mapsto P_X(B(x_1,\epsilon))$ is continuous and takes the value $0$ at $\epsilon=0$.
Hence, by the Intermediate Value theorem, there must exist an $\epsilon_3 \in (0,
\epsilon_2)$ such that $P_X\lp B(x_1,\epsilon_3)\rp = P_X\lp B(x_2, \epsilon_2)\rp$.

\item  We  now define $G_{\Delta} = \big(G_{\Delta}^*\setminus B(x_2,\epsilon_2) 
\big) \cup B(x_1,\epsilon_3)$, $G_1 = (G_1^*\setminus B(x_1,\epsilon_3))\cup \{x
\in B(x_2,\epsilon_2): \eta(x)-1/2\geq 0\}$, and $G_{-1} = (G_{-1}^*\setminus
B(x_1,\epsilon_3))\cup \{ x\in B(x_2,\epsilon_2): \eta(x)-1/2<0\}$. We first note 
that this new classifier $g$ is feasible for \eqref{bayes_optimal_delta} as 
$P_X(G_{\Delta}) = P_X(G_{\Delta}^*)-P_X(B(x_2,\epsilon_3)) + P_X(B(x_1,\epsilon_2))
= P_X(G_{\Delta}^*) \leq \delta$. 

To compute the excess risk of $g$ over $g^*$, we need to introduce some notation. 
Define $F_1^* = G_1^*\setminus B(x_1,\epsilon_3)$, $F_{-1}^* = G_{-1}^*\setminus
B(x_1,\epsilon_3)$, $E_1^* =G_1^*\cap B(x_1,\epsilon_3)$, $E_{-1}^* = G_{-1}^*\cap 
B(x_1,\epsilon_3)$, $U_1= \{x \in B(x_2,\epsilon_2)\ \mid \ \eta(x)\geq 1/2\}$, 
and $U_{-1} = \{x \in B(x_2,\epsilon_2)\ \mid \ \eta(x)<1/2\}$. Then we have  
\begin{align*}
   R(g) - R(g^*) &= \int_{F^*_{-1}\cup U_{-1}}\eta(x)dP_X(x) + \int_{F^*_1\cup U_1}(1
   -\eta(x)dP_X(x) \\
   &- \int_{F^*_{-1}\cup E_{-1}^*}\eta(x)dP_X(x) - \int_{F^*_1\cup E_1^*}(1-\eta(x))dP_X(x) \\
   & = \bigg(\int_{U_{-1}}\eta(x)dP_X(x) - \int_{E_{-1}^*}\eta(x)dP_X(x)\bigg) \\ 
   &+ \bigg(\int_{U_1}(1 -  \eta(x))dP_X(x) - \int_{E_1^*}(1-\eta(x))dP_X(x)\bigg) \\
 &\coloneqq \tau_1 + \tau_2
\end{align*}
Now, we consider the two terms separately:
\begin{align*}
    \tau_1 &\leq \sup_{x \in U_{-1}}\eta(x)P_X(U_{-1}) - \inf_{x \in E_{-1}^*}\eta(x)
    P_X(E_{-1}^*) \\
    &= \lp \frac{1}{2} - \inf_{x\in U_{-1}}\lp \frac{1}{2} -\eta(x)\rp\rp P_X(U_
    {-1}) - \lp \frac{1}{2} - \sup_{x \in E_{-1}^*}\lp \frac{1}{2} - \eta(x)\rp
    \rp P_X(E_{-1}^*)\\
    & = \frac{1}{2}\lp P_X(U_{-1}) - P_X(E_{-1}^*)\rp - \inf_{x\in U_{-1}}\lp \frac{1}
    {2} - \eta(x)\rp P_X(U_{-1}) + \sup_{x \in E_{-1}^*}\lp \frac{1}{2}-\eta(x)\rp
    P_X(E_{-1}^*).
\end{align*}
\begin{align*}
    \tau_2 & \leq \sup_{x\in U_1}\lp 1-\eta(x)\rp P_X(U_1) - \inf_{x\in E_1^*}\lp 
    1-\eta(x)\rp P_X(E_1^*).\\
    & = \lp \frac{1}{2} - \inf_{x \in U_1}\lp \eta(x) - \frac{1}{2}\rp \rp P_X(U_1)
- \lp \frac{1}{2} - \sup_{x \in E_1^*}\lp \eta(x) - \frac{1}{2}\rp \rp P_X(E_1^*)\\
& = \frac{1}{2}\lp P_X(U_1) - P_X(E_1^*)\rp - \inf_{x \in U_1}\lp \eta(x) - \frac{1}
{2}\rp + \sup_{x\in E_1^*}\lp \eta(x) - \frac{1}{2} \rp 
\end{align*}
Next, we note the following: 
\begin{align*}
\min \left\{ \inf_{x\in U_{-1}}\lp \frac{1}{2}-\eta(x)\rp,\ \inf_{x\in U_1}\lp 
\eta(x) - \frac{1}{2}\rp \right\} &\geq \inf_{x \in B(x_2,\epsilon_3)}\left| \eta(x)
-\frac{1}{2} \right|\\
\max\left\{\sup_{x\in E_{-1}^*}\lp \frac{1}{2}-\eta(x)\rp , \sup_{x\in E_1^*}\lp
\eta(x)-\frac{1}{2}\rp \right\} & \leq \sup_{x \in B(x_1,\epsilon_2)}\left| \eta(x)
-\frac{1}{2}\right|. 
\end{align*}
Combining these observations, we get the following:
\begin{align*}
    \tau_1 + \tau_2 &\leq \frac{1}{2}\lp P_X(U_{-1}) + P_X(U_1) - P_X(E_{-1}^*) - 
    P_X(E_1^*)\rp \\
    & - \lp  P_X(U_{-1}) + P_X(U_1)\rp \inf_{x \in B(x_2,\epsilon_3)}\left|\eta(x)-
   \frac{1}{2}  \right|\\
   & + \lp P_X(E_{-1}^*) + P_X(E_1^*)\rp \sup_{x \in B(x_1,\epsilon_2)}\left| \eta(x)
   -\frac{1}{2}\right |.
\end{align*}
Finally, by construction we have $P_X(U_{-1}) + P_X(U_1) = P_X(B(x_2,\epsilon_3)) =
P_X(B(x_1,\epsilon_2)) = P_X(E_{-1}^*) + P_X(E_1^*) \coloneqq \Gamma >0$. This gives us
\begin{align*}
    \tau_1 + \tau_2 & \leq \Gamma \lp \sup_{x \in B(x_1,\epsilon_2)}\left|\eta(x)-
    \frac{1}{2}\right| - \inf_{x \in B(x_2,\epsilon_3)}\left| \eta(x) - \frac{1}{2}
   \right|\rp \stackrel{(a)}{<} 0,
\end{align*}
where {\bf (a)} follows from \eqref{eq:inf_minus_sup}. 
This implies that the classifier $g$ is feasible for \eqref{bayes_optimal_delta} and 
has strictly smaller risk than $g^*$, thus contradicting the assumption of optimality
of $g^*$. 
\end{itemize}
\end{proof}

%-------------------------------------------------------------------
% APPENDIX A.2
%-------------------------------------------------------------------

\subsection{Proof of Theorem~\ref{theorem:bayes_optimal}}
\label{appendix:claim2}
\begin{proofoutline}
For any randomized classifier $g = \lp g_{-1},g_{1},\gD \rp$ that satisfies the constraint of~\eqref{bayes_optimal_delta}, we may write $R(g) = \int_{\X}\eta(x) g_{-1} + (1-\eta(x))g_1 dP_X$ and $\int_{\X}g_{\Delta}dP_X \leq \delta$. Since the five sets $(G_{-1}^*,G_1^*,\Gds,\partial G_1*,\partial G_{-1}^*)$ partition $\X$,  we may obtain a representation of $R(g^*)$ as the sum of the integrals over these five disjoint sets. The rest of the proof proceeds by employing the definition of $g^*$ to show that $R(g)- R(g^*)$ is non-negative, for any feasible abstaining classifier $g$. We have also included a separate proof for the case of deterministic classifiers, as this case is easier to follow (than the more general stochastic case) and we will employ similar arguments in the proofs of Theorem~\ref{theorem:excess_risk_bound} and Theorem~\ref{theorem:binary_search} later on in the paper. 
%\todo{Discuss the deterministic and stochastic proofs.}
\end{proofoutline}

\begin{proof}
Given any randomized feasible classifier $g = \lp \gmo, g_1, \gD \rp$, we can  write the
 excess risk $R(g) - R(g^*)$ as 
\begin{equation*}
R(g) - R(g^*) = \int_{\X} \lp \eta \gmo + (1-\eta)g_1 \rp dP_X - R(g^*)
\end{equation*}

We introduce the notation $f(x) = \eta(x)\gmo(x) + (1-\eta(x))g_1(x)$. Now, by 
the definition of $g^*$, we obtain the following:

\begin{align*}
R(g) - R(g^*) &= \int_{G_{1}^*}\lp f - (1-\eta)  \rp dP_X + \int_{G_{-1}^*}\lp f -
\eta \rp dP_X + \int_{\Gds} f  dP_X \\
&+ \int_{\partial G_1^*}\lp f + c_0 (1-\eta) \rp dP_X - \int_{\partial G_{-1}^*}
\lp f -c_0\eta \rp dP_X \\ 
&\coloneqq T_1 + T_2 + T_3 + T_4 + T_5
\end{align*}
where the terms $T_i$, for $i=1,\ldots,5$, are defined implicitly.
Using the notation $\lambda = \lp \frac{1}{2} - \gd \rp$, we now bound these five 
terms separately as follows :
\begin{itemize}
    \item $T_1 \geq -\lambda \int_{G_1^*}\gD dP_X$. To get this, we first  use
     the fact that $\eta \geq 1-\eta$ in the set $G_1^*$, which implies that 
     $f \geq (1-\gD)(1-\eta)$. Finally, the result follows from the fact that 
     $(1-\eta) \leq \lambda$ in the set $G_1^*$.  

    \item $T_2 \geq -\lambda \int_{G_{-1}^*}\gD dP_X$ follows from the fact that
    $f \geq (1-\gD)\eta$, and $\eta \leq \lambda$ in the set $G_{-1}^*$. 

    \item $T_3 \geq -\lambda \int_{\Gds}\gD dP_X + \lambda P_X\lp \Gds \rp$ follows from 
    the fact that $\min \{ \eta, 1-\eta\} \geq \lambda$ in $\Gds$, which implies
    that $f \geq \lambda (1-\gD)$ on $\Gds$.  
    
    \item $T_4 \geq \lambda c_0 P_X \lp \partial G_1^* \rp - \lambda \int_{
    \partial G_1^*}\gD dP_X$ follows from the fact that $\eta = 1-\lambda$ on 
    the set $\partial G_1^*$.
    
     \item $T_5 \geq \lambda c_0 P_X \lp \partial G_{-1}^* \rp - \lambda \int_{
    \partial G_{-1}^*}\gD dP_X$ follows from the fact that $\eta = \lambda$ on 
    the set $\partial G_{-1}^*$. 

    \end{itemize}

Combining these observations, we obtain
\begin{align*}
R(g) - R(g^*) & \geq -\lambda \int_{\X} \gD dP_X + \lambda P_X\lp \Gds \rp + 
\lambda c_0 P_X \lp \partial G_1^* \cup \partial G_{-1}^* \rp \\
& \stackrel{(a)}{=} \lambda \lp \delta - \int_{\X} \gD dP_X \rp \\
& \stackrel{(b)}{\geq} 0, 
\end{align*}
where {\bf (a)} follows from the choice of the term $c_0$ and {\bf (b)} follows 
from the assumption that $g$ is a feasible randomized classifier for the problem
~\eqref{bayes_optimal_delta}. 
%

%------------------------------------------------------------------------------
% ALTERNATIVE PROOF FOR THE NON RANDOMIZED CASE
%------------------------------------------------------------------------------

\paragraph{Alternate proof for the non-randomized case:}
In the case where we have $P_X\lp \Gds \rp = \delta$, the terms $T_4$ and $T_5$
are zero since $P_X\lp \partial G_1^*\rp = P_X \lp \partial G_{-1}^* \rp  = 0$.
The optimal classifier does not require randomization in these situations.

Since, we will use similar arguments for the proofs of Theorem~\ref{theorem:excess_risk_bound} and Theorem~\ref{theorem:binary_search}
, for completeness, we now provide the steps of a proof of the optimality of the 
classifier $g^* = (G_{-1}^*, G_1^*, \Gds)$ when restricted to the class of 
deterministic classifiers. 

\begin{proof}
Given any feasible (satisfies the constraint of~\eqref{bayes_optimal_delta}) classifier $g = \lp G_{-1}, G_1, \Gd \rp$, we can write the excess risk $R(g) - R(g^*)$ as 
\begin{equation*}
    R(g) - R(g^*) = \int_{G_{-1}}\eta(x) dP_X + \int_{G_1}\lp 1-\eta(x)\rp dP_X - \int_{G_{-1}^*}\eta(x) dP_X - \int_{G_1^*}\lp 1-\eta(x)\rp dP_X. 
\end{equation*}
Since we have $G_i = (G_i \cap G_{-1}^*)\cup (G_i\cap G_1^*)\cup (G_i\cap \Gds)$ and $G_i^* = (G_i^*\cap G_{-1})\cup (G_i^*\cap G_1) \cup (G_i^*\cap \Gd)$, for $i\in\{-1,1,\Delta\}$, we obtain
\begin{align*}
R(g) - R(g^*) &= \int_{G_{-1}\cap G_1^*}\lp 2\eta(x) -1 \rp dP_X + \int_{G_{-1}\cap\Gds}\eta(x) dP_X + \int_{G_1\cap G_{-1}^*}\lp 1-2\eta(x) \rp dP_X \\
&+ \int_{G_1\cap \Gds}\lp 1-\eta(x) \rp dP_X - \int_{G_{-1}^*\cap \Gd}\eta(x) dP_X - \int_{G_1^*\cap \Gd}(1-\eta(x))dP_X  \\
&\coloneqq L_1 + L_2 + L_3 + L_4 - L_5 - L_6,
\end{align*}
where the terms $L_i$, for $i=1,\ldots,6$, are defined implicitly. We now bound these six terms separately as follows:
\begin{itemize}
    \item $L_1 > 2\gd P_X\lp G_{-1} \cap G_1^* \rp$, since $\eta(x) > 1/2 + \gd$ on the set $G_1^*$. 
    \item $L_2 \geq \lp 1/2 - \gd \rp P_X\lp G_{-1} \cap \Gds \rp$, since $\eta(x) \geq 1/2 - \gd$ on the set $\Gds$.
    \item $L_3 > 2\gd P_X\lp G_1\cap G_{-1}^*\rp$, since $\eta(x) < 1/2 - \gd$ on the set $G_{-1}^*$. 
    \item $L_4 \geq \lp 1/2 - \gd \rp P_X\lp G_1 \cap \Gds\rp$, since $\eta(x) \leq 1/2 + \gd$ on the set $\Gds$. 
    \item $L_5 < \lp 1/2 - \gd \rp P_X\lp G_{-1}^*\cap \Gd \rp$, since $\eta(x) < 1/2 - \gd$ on set $G_{-1}^*$. 
    \item $L_6 < \lp 1/2 - \gd  \rp P_X \lp G_1^*\cap \Gd \rp$, since $\eta(x) > 1/2 + \gd$ on the set $G_1^*$. 
\end{itemize}

Combining these observations, we obtain
\begin{align}
\label{eq:temp0}        
        R(g) - R(g^*) &\geq 2\gd \big(P_X(G_{-1} \cap G_1^*) + P_X(G_1 \cap G_{-1}^*)\big) + \lp \frac{1}{2} - \gd \rp P_X\big( \Gds \cap (G_{-1}\cup G_1)\big) \nonumber \\
        &- \lp \frac{1}{2} - \gd \rp P_X\big( \Gd \cap ( G_{-1}^*\cup G_1^*) \big) \nonumber \\
        &\stackrel{(a)}{=} 2\gd \big(P_X(G_{-1} \cap G_1^*) + P_X(G_1 \cap G_{-1}^*) \big) + (1 - 2\gd) \big( P_X(\Gds) - P_X( \Gd)\big) \nonumber \\
        &\stackrel{(b)}{=} 2\gd \big(P_X(G_{-1} \cap G_1^*) + P_X(G_1 \cap G_{-1}^*) \big) + (1 - 2\gd) \big( \delta - P_X(\Gd) \big) \\
        &\stackrel{(c)}{\geq} 0. \nonumber 
\end{align}
{\bf (a)} comes from the fact that $P_X\big( \Gds \cap (G_{-1}\cup G_1)\big) = P_X(\Gds) - P_X(\Gd)$ and $P_X\big( \Gd \cap (G^*_{-1}\cup G^*_1)\big) = P_X(\Gd) - P_X(\Gds)$. \\
{\bf (b)} comes from the assumption that the classifier $g^* = \lp G_{-1}^*, G_1^*,
\Gds \rp$ satisfies the constraint of~\eqref{bayes_optimal_delta} with equality, and thus, $P_X(G^*_\Delta)= \delta$. \\
{\bf (c)} comes from the fact that neither of the two terms in~\eqref{eq:temp0} can be negative. 
\end{proof}

\end{proof}

%-------------------------------------------------------------------
%-------------------------------------------------------------------
% APPENDIX B: Proofs for Plug-in Classifier
%-------------------------------------------------------------------
%-------------------------------------------------------------------

\newpage
\section{Details from Section~\ref{sec:plugin_classifier}}
%------------------------------------------------------------------------------
% PROOF OF CONCENTRATION RESULT
%------------------------------------------------------------------------------
\subsection{Concentration Results}
\label{appendix:concentration}

\begin{proposition}
\label{prop:concentration}
Define $N = \Bigl\lfloor \lp\frac{n\mu_{\min}}{16\log n}\rp^{\frac{1}{D}} \Bigr\rfloor$ and $H = \{\frac{1}{N}, \frac{2}{N},\ldots, 1\}$. Then, we have the following:
\begin{enumerate}
\item Event $\Omega_1 = \cap_{h \in H}\Omega_{1,h}$ occurs with probability at least $1-1/n$, where $\Omega_{1,h}$ is defined as 
\begin{align*}
\Omega_{1,h} \coloneqq \bigg\{ & |\hat{\eta}_h(x_{h,i}) - \bar{\eta}(E_{h,i})|\leq \sqrt{\frac{2\log(2n/h^D)}{n(E_{h,i})}}, 
\forall i \in[M_h]\bigg\},
\end{align*}
where $x_{h,i}$ is any point in the cell $E_{h,i}$. Note that $\hat{\eta}_h(\cdot)$ returns the same value for all $x\in E_{h,i}$.
\item Event $\Omega_2 = \cap_{h \in H}\Omega_{2,h}$ occurs with probability at least $1-1/n$, where $\Omega_{2,h}$ is defined as
\begin{equation*}
\Omega_{2,h} \coloneqq \bigg\{ n(E_{h,i}) \geq \lp \frac{n\mu_{\min}h^D}{4}\rp,\; \forall i \in [M_h]\bigg\}.
\end{equation*}
\end{enumerate}
\end{proposition}

\begin{proofoutline}
The proof of the first statement uses the multiplicative form of the Chernoff bound (Eq.~7 in~\citealt{hagerup1990guided}), while the second statement follows by employing the Hoeffding's inequality. The detailed proof is provided in Appendix~\ref{appendix:concentration}.
\end{proofoutline}

\begin{proof}
\begin{enumerate}
    \item If we show that the event $\Omega_{1,h}^c$ occurs with probability at most $1/(nN)$, then the final statement follows by a union bound over $N$ members of the set $H$. In order to show that $P(\Omega_{1,h}^c)\leq 1/(nN)$, first consider the event $\mathcal{E}_{h,i}(k) \coloneqq \{n(E_{h,i})=k\}$, for $k=0,1,\ldots,n$. Then, by the Hoeffding's inequality, for any $t_k>0$, we have
\[
    P\bigg( |\hat{\eta}(x_{h,i}) - \bar \eta (E_{h,i})| > t_k \bigg| \mathcal{E}_{h,i}(k)
\bigg) \leq 2e^{-k t_k^2/2}.
\]
Now, with $t_k = \sqrt{2\frac{\log(2nN\lceil1/h\rceil^D)}{k}}$, we have $2e^{-k t_k^2/2}=\frac{1}{Nn}\lceil1/h\rceil^D$, and thus, we may write
\begin{align*}
    &P\lp( |\hat\eta(x_{h,i}) - \bar\eta(E_{h,i})| >
    \sqrt{\frac{2\log(2nN\lceil1/h\rceil^D)}{n(E_{h,i})}} \rp = \\
    &\sum_{k=0}^nP\lp
    \mathcal{E}_{h,i}(k)\rp P\bigg( |\hat{\eta}(x_{h,i}) - \bar \eta (E_{h,i})| > t_k
    \bigg| \mathcal{E}_{h,i}(k)\bigg) \leq \\
    &\sum_{k=0}^nP(\mathcal{E}_{h,i}(k))\frac{1}{Nn\lceil1/h\rceil^D} = \frac{1}{Nn \lceil1/h\rceil^D}.
\end{align*}
Thus, by taking a union bound over all the $M_h=\lceil1/h\rceil^D$ cubes $E_{h,1},\ldots,E_{h,M_h}$, we have $P(\Omega_{1,h}^c)\leq 1/(nN)$. 
    \item If we show that the event $\Omega_{2,h}^c$ occurs with probability at most $1/(nN)$, then the final statement follows by a union bound over $N$ members of the set $H$. In order to show that $P(\Omega_{2,h}^c)\leq 1/(nN)$, we first introduce the notation $\mu_{h,i} \coloneqq P_X(E_{h,i})$. By the multiplicative form of the Chernoff bound~\citep[(7)]{hagerup1990guided}, for any $c_{h,i} \in (0,1)$, we have
\[ 
           P_X\big( n(E_{h,i}) < (1 - c_{h,i}) n\mu_{h,i} \big) \leq e^{-\frac{n\mu_{h,i} c_{h,i}^2}{2}}. 
\]
To complete the proof, it is sufficient to choose $c_{h,i}$ to ensure that $e^{-\frac{n\mu_{h,i}c_{h,i}^2}{2}} \leq \frac{1}{nNM_h}$ and $(1 - c_{h,i}) n\mu_{h,i}\geq \frac{n\mu_{\min} h^D}{4}$. Now we prove that a suitable choice of $c_{h,i}$ for obtaining these two inequalities is $c_{h,i} = \sqrt{ \frac{ 6 \log (n\mu_{\min})}{n \mu_{h,i}}}$. We start by showing
\begin{align*}
    \frac{n \mu_{h,i} c_{h,i}^2}{2} &= 3 \log \lp n \mu_{ \min} \rp \stackrel{\text{(a)}}{\geq} \log \lp n^{(2D+1)/D} \mu_{\min}^{(D+1)/D} \rp 
    = \log \lp n (n\mu_{\min})^{(D+1)/D} \rp \\
    &\stackrel{\text{(b)}}{\geq} \log \lp n N^{D+1} \rp \stackrel{\text{(c)}}{\geq} \log \lp nNM_h \rp. 
\end{align*}
{\bf (a)} comes from the fact that for $D\geq 1$, we have $\frac{D+1}{D}\leq \frac{2D+1}{D}\leq 3$. \\
{\bf (b)} follows from the fact that from the statement of the proposition, we have $N =
\left(\frac{n\mu_{\min}}{16\log n}\right)^{\frac{1}{D}}$, and thus, $n \mu_{\min} \geq
\frac{n\mu_{\min}}{16\log n} = N^D$. \\
{\bf (c)} comes from the fact $h\geq \frac{1}{N}$ (see the statement of the proposition) and 
$M_h=\lceil1/h\rceil^D$. 

This completes the prove of $e^{-\frac{n\mu_{h,i}c_{h,i}^2}{2}} \leq \frac{1}{nNM_h}$. 

We now move on to prove $(1 - c_{h,i}) n\mu_{h,i}\geq \frac{n\mu_{\min} h^D}{4}$. We start by showing
\begin{align*}
    c_{h,i}^2 &= \frac{6 \log \lp n \mu_{\min}\rp}{n \mu_{h,i}} \stackrel{\text{(a)}}{\leq} \frac{6 \log \lp n \mu_{\min}\rp}{n \mu_{\min}h^D} \stackrel{\text{(b)}}{\leq} \frac{ 6 \log \lp n \mu_{\min}\rp N^D  }{n \mu_{min}} \\ 
    &\stackrel{\text{(c)}}{=} \frac{ 6 \log \lp n \mu_{\min} \rp  n \mu_{min}}{ n \mu_{\min} 16 \log (n) } = \frac{ 6 \log \lp n \mu_{\min} \rp}{16 \log (n) } \stackrel{\text{(d)}}{\leq} \frac{6}{16} \\
    \Longrightarrow c_{h,i} &\leq \sqrt{\frac{6}{16}} \leq \frac{3}{4}.
\end{align*}
{\bf (a)} comes from the fact that $\mu_{h,i}\geq\mu_{\min}h^D$. \\
{\bf (b)} comes from the fact $h\geq \frac{1}{N}$ from the statement of the proposition. \\
{\bf (c)} comes from the definition of $N$ from the statement of the proposition. \\
{\bf (d)} comes from the fact that $\mu_{\min}\leq 1$.

Thus, we have  under the event $\Omega_{2,h}$, we have 
\begin{equation*}
    n(1-c_{h,i})\mu_{h,i} \geq n\lp 1- \frac{3}{4} \rp \mu_{\min}h^D = \frac{ n \mu_{\min}h^D}{4}.
\end{equation*}

%which implies that $n(E_{h,i}) \geq (1-c)n\mu_{h,i}$ for all $i$, with probability at least $\lp 1-\frac{e^{-\frac{n\mu_ic^2}{2}}}{h^D}\rp$. Now, since $h^D \geq \frac{16\log n}{n\mu_{\min}}$, we get the required result by selecting $c = \sqrt{\frac{4\log(nN)}{n \mu_{\min}}}$ and taking a union bound over all cubes $E_{h,1},\ldots,E_{h,M_h}$.  
\end{enumerate}
\end{proof}

%------------------------------------------------------------------------------
% ERROR OF THE PLUGIN ESTIMATOR
%------------------------------------------------------------------------------
\subsection{Proof of Proposition~\ref{prop:regression_estimator}}
\label{appendix:regression_estimator}

We drop the subscript $x$ for $h^*_x$, $h^*_{1,x}$, and $\hat{h}_x$ in this section. Introduce the following definitions:
\begin{align*}
h^* &= \max \big\{h \in (0,1) \mid \ e_S(h,x) \geq  e_D(h,x)\big\}, \qquad\qquad h^*_1= \max \big\{h \in H\ \mid \ e_S(h,x) \geq e_D(h,x)\big\}, \\
\hat{h}&= \max \big\{ h \in H\ \mid \ |\hat{\eta}_h(x)-\hat{\eta}_{h'}(x)| \leq 4e_S(h',x), \;\; \forall\ h'\in H,\; h'\leq h\big\}.
\end{align*}
Now we may write the following:
\begin{align}
\label{eq:temp-AppB-0}
|\hat{\eta}(x) - \eta(x)| = |\hat{\eta}_{\hat{h}}(x) - \eta(x)| &\leq |\hat{\eta}_{\hat{h}}(x) - \hat{\eta}_{h^*_1}(x)| + |\hat{\eta}_{h_1^*}(x)-\eta(x)| \nonumber \\
&\stackrel{\text{(a)}}{\leq} 4e_S(h_1^*,x) + |\hat{\eta}_{h_1^*}(x)-\eta(x)| \nonumber \\
&\stackrel{\text{(b)}}{\leq} 4e_S(h_1^*,x) + e_S(h_1^*,x) + e_D(h_1^*,x) \nonumber \\
&\stackrel{\text{(c)}}{\leq} 6e_S(h^*_1,x).
\end{align}
{\bf (a)} follows from the fact that $h_1^* \leq \hat{h}$, and thus,
$|\hat{\eta}_{\hat{h}}(x)-\hat{\eta}_{h_1^*}(x)|\leq 4e_S(h_1^*,x)$. Now what is left to show is that $h_1^* \leq \hat{h}$. To see this, we first define the set $\hat{H} = \{h \in H \ \mid \ |\hat{\eta}_h(x)-
\hat{\eta}_{h'}(x)| \leq 4e_S(h',x), \;\; \forall\ h'\in H,\; 
h'\leq h\big \}$. Clearly, $\hat{h}$ is the maximum element in $H$, from its definition. Thus, to show that $h_1^* \leq \hat{h}$, it suffices to prove that $h^*_1\in\hat{H}$. Consider any $h'\leq h_1^*,\;h' \in H$. We then have the following:
\begin{align*}
|\hat{\eta}_{h_1^*}(x) - \hat{\eta}_{h'}(x)| &\leq |\hat{\eta}_{h_1^*}(x) - \eta(x)| + |\hat{\eta}_{h'}(x) - \eta(x)| \\
&\leq e_S(h_1^*, x) + e_D(h_1^*,x) + e_S(h',x) + e_D(h',x) \\
&\leq 2e_S(h_1^*,x) + 2e_S(h',x) \leq 4e_S(h',x), 
\end{align*}
which implies that $h_1^*\in \hat H$. Note that the last inequality comes from the fact that $e_S(h,x)$ decreases as $h$ is increased. \\
{\bf (b)} is from Eq.~\ref{eq:estimation-error}. \\
{\bf (c)} uses the definition of $h_1^*$. 

From the definitions of $h^*$ and $h^*_1$, we have $h^*-1/N \leq h_1^* \leq h^*$, and thus, we may write for $D \geq 2$
\begin{align}
\label{eq:temp-AppB-1}
e_S(h^*_1,x) &= e_S(h^*,x)\lp \frac{h^*}{h_1^*}\rp^{D/2} \leq e_S(h^*,x)\lp \frac{h^*}{h^*-\frac{1}{N}}\rp^{D/2} \nonumber \\
&= e_S(h^*,x) \lp \frac{1}{1-\frac{1}{Nh^*}}\rp^{D/2} \stackrel{\text{(b)}}{\leq} e_S(h^*,x)\lp \frac{1}{1-\frac{D}{2Nh^*}}\rp \nonumber \\
&\stackrel{\text{(c)}}{\leq} \frac{4}{3} e_S(h^*, x).
\end{align}
{\bf (b)} follows from  the Bernoulli's inequality $(1+x)^z\geq 1+xz$, for $x\geq -1$ and $z\geq 1$.\\ 
{\bf (c)} relies on the assumption that $n$ is large enough to ensure that $h^*N \geq 2D$. 

The case $D=1$ can be handled similarly as 
\begin{align}
    e_S(h_1^*,x) &\leq e_S(h^*,x) \lp \frac{1}{1 - \frac{1}{Nh^*}}\rp^{1/2} \leq e_S(h^*,x) \lp \frac{1}{1 -\frac{1}{2}} \rp^{1/2} \nonumber \\
    &= \sqrt{2} e_S(h^*,x) 
    \label{eq:temp-AppB-2} 
\end{align}

Since $\max\{4/3, \sqrt{2}\}\leq 3/2$, the final result follows from~\eqref{eq:temp-AppB-0},~\eqref{eq:temp-AppB-1}, and~\eqref{eq:temp-AppB-2}, i.e.,~under the events $\Omega_1$ and $\Omega_2$, for all $x \in \X$, we have $|\hat{\eta}(x) - \eta(x)| \leq 9 e_S(h^*,x)$. 
    
We now further assume that the regression function $\eta(\cdot)$ is H\"older continuous with exponent
$0<\beta\leq 1$. Since $h$ has the opposite effect on the two error terms $e_S(h,x)$ and $e_D(h,x)$, its
optimal value that minimizes the upper-bound $e_S(h,x)+e_D(h,x)$ is obtained by putting these two terms
equal, i.e.,~$e_S(h,x)=e_D(h,x)$. Putting $e_S(h,x)=\sqrt{ \frac{ 32\log(n\mu_{\min}) }{ n\mu_{\min} h^D } }$ equal
to $e_D(h,x)=L(\sqrt{D}h)^\beta$ will give us $h^*=\left(\frac{32\log(n\mu_{\min})}{n\mu_{\min} L^2D^\beta}
\right)^{\frac{1}{2\beta+D}}$. Now plugging the value of $h^*$ in $e_D(h^*,x)=e_S(h^*,x)$ and using what
we proved in the first part of this proposition, i.e.,~$|\hat{\eta}(x) - \eta(x)| \leq 9 e_S(h^*,x)$, we
have
\begin{align*}
|\hat{\eta}(x)-\eta(x)| &\leq 9e_S(h^*,x) = 9e_D(h^*,x) \leq 9L(\sqrt{D}h^*)^\beta \\ 
&= 9L^{\frac{D}{2\beta+D}} D^{\frac{\beta D}{2(2\beta+D)}}\left(\frac{32\log(n\mu_{\min})}{\mu_{\min}}\right)^{\frac{\beta}{2\beta+D}} n^{\frac{-\beta}{(2\beta+D)}} = b_n = \tilde{\mathcal{O}}(n^{-\beta/(2\beta+D)}).
\end{align*} 
    
%------------------------------------------------------------------------------
% PROOF OF SLACK TERM 
%------------------------------------------------------------------------------
\subsection{Proof of Proposition~\ref{prop:slack}}
\label{appendix:slack}

For this inequality, we first note that the class of functions $\mathcal{F}_1 \coloneqq \{f_c:\mbb{R}\to \{0,1\}\; \mid \; f_c(x) =
\indi_{\{|x|\leq c \}},\ c \in \mbb{R}\}$,  has the VC dimension of $2$~\citep[\S~6.3.2]{shalev2014understanding}. This implies the following uniform convergence result with probability at least $1-1/m$, for $m$ samples $\{Z_j\}_{j=1}^m$ drawn i.i.d.~from any distribution $P_Z$~\citep[\S~28.1]{shalev2014understanding}:
\begin{align}
\sup_{f_c \in \mathcal{F}_1 } \lp \frac{1}{m}\left|\sum_{j=1}^m \Big( f_c(Z_j) - \mbb{E}_{P_Z}\big[ f_c(Z_j)\big] \Big)\right| \rp &\leq 2\sqrt{\frac{16 \log(em/2) + 2\log(4m)}{m} } \nonumber\\
&\leq 2\sqrt{\frac{18 \log(4m)}{m} } \coloneqq a_m. \label{eq:temp-AppB-3}
\end{align}
%
%In particular, we can choose the random variables $Z_j=\hat{\eta}(X_{n+j})$, for $j=1,\ldots,m$, to obtain the required statement.
%\todo{The proof is rough, needs more elaboration. Several steps are missing.}

Now, we note that conditioned on the labelled training set $S_l$, the estimator $\hat{\eta}(\cdot)$ is a fixed function, and $\{X_{n+j}\}_{j=1}^m$ are independent of the samples in $S_l$. We define the random variables $\big\{Z_j = \hat{\eta}(X_{n+j})\big\}_{j=1}^m$ and introduce the event
\[
\mathcal{E}  = \left \{ \sup_{f_c \in \mathcal{F}_1} \lp  \frac{1}{m} \left|\sum_{j=1}^m
\Big(f_c(Z_j) - E_{P_Z}\big[ f_c(Z_j)\big]\Big)\right|  \leq a_m \rp \right \}.
\]
Then, we have
\begin{equation*}
P\lp \mathcal{E}^c \rp = \mbb{E}\lb \indi_{\mathcal{E}^c} \rb 
= \mbb{E}\big[ \mbb{E}[ \indi_{\mathcal{E}^c} \big| S_l ] \big] = \mbb{E}\lb P( \mathcal{E}^c \big| S_l ) \rb \leq \mbb{E} \lb \frac{1}{m} \rb  = \frac{1}{m},
\end{equation*}
where the  inequality follows from~\eqref{eq:temp-AppB-3}. This proves that $P\lp \mathcal{E} \rp = P(\Omega_3) \geq 1 - \frac{1}{m}$.

%--------------------------------------------------------------------------------
%Proof of Theorem 2
%--------------------------------------------------------------------------------

\subsection{Proof of Theorem~\ref{theorem:excess_risk_bound}}
\label{appendix:excess_risk_bound}

\begin{proof}
We assume that the events $\Omega_1$, $\Omega_2$, and $\Omega_3$ occur, whose probability is at least $1 - 1/m - 2/n$ (see Propositions~\ref{prop:concentration} and~\ref{prop:slack}).

We first present a lemma which tells us that the estimated threshold $\hat{\gamma}$ is close to the true threshold $\gd$. 

\begin{lemma}
    \label{lemma:risk0}
    Suppose $m$ is large enough to ensure that $\lp \frac{a_m}{C_1}\rp^{1/\rho_1} \leq b_n$. Then we have 
    \begin{equation}
        \label{eq:gamma_estimate}
        \gd - 4b_n \leq \hat{\gamma} \leq \gd + b_n. 
    \end{equation}
\end{lemma}
\begin{proof}
We first prove the upper bound on $\hat{\gamma}$. 
\begin{align*}
    \hat{P}_m \lp |\hat{\eta}-1/2| \leq \hat{\gamma} \rp & \leq \delta - a_m \; \stackrel{(i)}{\Rightarrow} P_X\lp |\hat{\eta}-1/2|\leq \hat{\gamma} \rp \leq \delta \\
    \stackrel{(ii)}{\Rightarrow} P_X\lp |\eta - 1/2|\leq \hat{\gamma}-b_n\rp & \leq \delta\; = P_X\lp |\eta - 1/2|\leq \gd \rp \\
\; \Rightarrow \hat{\gamma}  \leq \gd + b_n. 
\end{align*}
In the above display, \textbf{(i)} follows from Proposition~\ref{prop:slack}, and \textbf{(ii)} follows from Proposition~\ref{prop:regression_estimator}.

Next, we observe that by the piecewise constant nature of the regression function estimator, the maximum difference in the $\hat{\eta}$ values of any two adjacent cells of the grid is no more than $2b_n$. This fact, coupled with the definition of $\hat{\gamma}$ implies that $\hat{P}_m \lp |\hat{\eta} - 1/2| \leq \hat{\gamma} + 2b_n \rp \geq \delta - a_m$. Thus we have the following series 
\begin{align*}
    \hat{P}_m \lp |\hat{\eta}-1/2| \leq \hat{\gamma} + 2b_n \rp & \geq \delta - a_m \\
    \Rightarrow P_X\lp |\eta - 1/2| \leq \hat{\gamma} + 3b_n \rp & \geq \delta - 2a_m \\
\end{align*}
Now, to obtain the lower bound on $\hat{\gamma}$, we observe that by the assumption~\ref{assump:detect}, the following is true.  
\begin{align*}
    P_X\lp |\eta - 1/2| \leq \gd - \lp \frac{a_m}{C_1}\rp^{1/\rho_1} \rp & \leq \delta - 2C_1 \lp \lp \frac{a_m}{C_1}\rp^{1/\rho_1}\rp^{\rho_1} = \delta - 2a_m. 
\end{align*}
The result follows by using the assumption that $b_n \geq \lp \frac{a_m}{C_1}\rp^{1/\rho_1}$. 
\end{proof}

%-------------------------------------------------------------------------------

{\bf Part~1:}
We first show that $P_X\lp \hat{g}(X) = \Delta \rp \leq \delta$, i.e., the constructed classifier $\hat{g}$ is feasible for~\eqref{bayes_optimal_delta}.
We consider two cases:
\begin{itemize}
    \item Case~1: $\hat{p}_1 \geq \delta - 5a_m$. In this case, the classifier does not randomize, and we have $P_X\lp \hat{g}(X) = \Delta \rp = P_X\lp \hat{G}_\Delta \rp \leq \hat{P}_m \lp \hat{G}_\Delta \rp + a_m \leq \delta$. 
    
    \item Case~2: $\hat{p}_1 < \delta - 5a_m$. In this case, due to randomization we have $P_X\lp \hat{g}(X) = \Delta\rp = P_X\lp \hat{G}_\Delta \rp + \hat{c}P_X\lp \partial \hat{G}_{-1}\cup \partial \hat{G}_1\rp$. By the definition of $\hat{c}$, we have 
    \begin{align*}
        P_X\lp \hat{g}(X) = \Delta \rp & \leq (\hat{p}_1 + a_m)  + \lp \frac{\delta- 5a_m - \hat{p}_1}{\hat{p}_2 - \hat{p}_1 - 2a_m} \lp \hat{p}_2 - \hat{p}_1 + 2a_m \rp \rp  \\
        & \leq \hat{p}_1 + a_m + (\delta - 5 a_m - \hat{p}_1) + \frac{(\delta - 5a_m - \hat{p}_1)4a_m}{\hat{p}_2 - \hat{p}_1 - 2a_m}\\
        & = \delta- 4a_m +  \frac{(\delta - 5a_m - \hat{p}_1)4a_m}{\hat{p}_2 - \hat{p}_1 - 2a_m}\leq \delta, 
    \end{align*}
    which completes the proof of the first part of Theorem~\ref{theorem:excess_risk_bound}. 
\end{itemize}

%-------------------------------------------------------------------------------
%
%-------------------------------------------------------------------------------
{\bf Part~2:} To prove the upper-bound on the excess risk, we first note that if we remove the randomization and deterministically declare $i$ in the region $\partial \hat{G}_i$ for $i \in\{-1,1\}$, the the excess risk can only increase. So, for the rest of this section, we will use $\hat{G}_i$ to represent the entire region in which label $i$ is declared, i.e., $\hat{G}_i \cup \partial \hat{G}_i$. 

Thus, we write the excess risk of the plug-in classifier $\hat{g}$ over the optimal classifier $g^*$ as %
\begin{align}
\label{eq:pf_eq1}
R(\hat{g}) - R(g^*) &= \int_{\hat{G}_{-1}}\eta(x) dP_X + \int_{\hat{G}_1}\big(1-\eta(x)\big)dP_X + \int_{\partial \hat{G}_{-1}}(1-\hat{c}  \\ 
& \quad - \int_{G_{-1}^*} \eta(x) dP_X - \int_{G_1^*}\big(1-\eta(x)\big)dP_X. 
\end{align}
Using the fact that $\hat{G}_{-1} = \hat{G}_{-1} \cap \big( G_{-1}^*\cup G_1^*\cup
G_{\Delta}^*\big)$, we may split the first term on the RHS of~\eqref{eq:pf_eq1} as
\begin{equation}
\label{eq:pf_eq2}
\int_{\hat{G}_{-1}}\eta(x) dP_X = \int_{\hat{G}_{-1}\cap G_{-1}^*}\eta(x) dP_X +
\int_{\hat{G}_{-1}\cap G_{\Delta}^*}\eta(x) dP_X + 
\int_{\hat{G}_{-1}\cap G_1^*}\eta(x) dP_X.
\end{equation}
Similarly, we may split all the other terms on the RHS of~\eqref{eq:pf_eq1} and obtain 
\begin{align}
R(\hat{g}) - R(g^*) &= \int_{\hat{G}_{-1}\cap G_{\Delta}^*}\eta(x) dP_X - \int_{\hat{G}_\Delta\cap G_{-1}^*}\eta(x) dP_X + \int_{\hat{G}_1\cap G_{\Delta}^*}\big(1-\eta(x)\big) dP_X \nonumber \\ 
&- \int_{\hat{G}_\Delta\cap G_1^*}\big(1-\eta(x)\big) dP_X + Q_1 + Q_2, 
\label{eq:pf_eq3}
\end{align}
where 
\begin{equation*}
    Q_1 = \int_{\hat{G}_{-1}\cap G_1^*}\big(2\eta(x)- 1\big) dP_X, \qquad\qquad Q_2 = \int_{\hat{G}_{1}\cap G^*_{-1}}\big(1-2\eta(x)\big) dP_X. 
\end{equation*}
%
% Thus, on subtracting (\ref{eq:pf_eq3}) from (\ref{eq:pf_eq2}), we get 
% \begin{align}
% \int_{\hat{G}_{-1}}\eta dP_X - \int_{G_{-1}^*}\eta dP_X &= \int_{\hat{G}_{-1}\cap G_{\Delta}^*}\eta dP_X - \int_{G_{-1}\cap \hat{G}_{\Delta}}\eta dP_X \\
% &= \int_{\hat{G}_{-1}\cap G_{\Delta}}\lp \eta - \frac{1}{2} + \gd \rp dP_X - \int_{G_{-1}^*\cap \hat{G}_{\Delta}}\lp \eta - \frac{1}{2} + \gd \rp dP_X  
% \end{align}

Now, we add and subtract $(1/2 - \gd)$ to the integrand of the first four terms on the RHS of~\eqref{eq:pf_eq3} and obtain 
\begin{equation}
R(\hat{g}) - R(g^*) = \left(\frac{1}{2} - \gamma_{\delta}\right) \left(P_X\lp G_{\Delta}^*\rp - P_X \lp \hat{g}(X)=\Delta \rp\right) + R_1 + R_2 + R_3 + R_4 + Q_1 + Q_2,
\label{eq:risk1}
\end{equation}
where
    \begin{align*}
        R_1 &= \int_{\hat{G}_{-1}\cap G_{\Delta}^*}\left( \eta(x) - \frac{1}{2} + \gamma_{\delta} \right)dP_X, \qquad R_2 = \int_{\hat{G}_{\Delta}\cap G_{-1}^*} \left(\frac{1}{2} - \gamma_{\delta} - \eta(x)\right) dP_X, \\
        R_3 & = \int_{\hat{G}_\Delta \cap G_1^*} \left(\eta(x) - \frac{1}{2} - \gd\right) dP_X, \qquad R_4 = \int_{\hat{G}_1\cap G_\Delta^*}\left(\frac{1}{2} + \gd - \eta(x) \right) dP_X.
    \end{align*}
%    
% Using Proposition~\ref{prop:slack}, the first term on the RHS of~\eqref{eq:risk1} can be bounded as
%
%----------------------------------------------------------------------------
% LEMMA: LOWER BOUND ON P_X(\hat{G}_\Delta)
%----------------------------------------------------------------------------
We now state a lemma that gives an upper-bound for the first term on the RHS
of~\eqref{eq:risk1}.

\begin{lemma}
\label{lemma:risk1}
We have $P_X\lp \hat{g}(X)=\Delta \rp \geq \delta - 5a_m$. 

Suppose $m$ is large enough to ensure that $C_1(\epsilon_0/4)^{\rho_1} > a_m$ and $n$ is large enough
to ensure that $b_n <\epsilon_0/2$. Then, we have $P_X(\hat{G}_\Delta) \geq \delta - 2a_m - 2C_0b_n^{\rho_0}$. 
\end{lemma}
\begin{proof}
We again have two cases:
\begin{itemize}
    \item  \textbf{Case~1:} $\hat{p}_1 \geq \delta - 5a_m$. In this case, there is no randomization, and by construction, we have $P_X\lp \hat{g}(X)=\Delta\rp \geq \delta - 5a_m$. 
    
    \item \textbf{Case~2:} $\hat{p}_1 <\delta -5a_m$. Here, we obtain a lower bound on the randomized classifier
    \begin{align*}
        P_X\lp \hat{g}(X)=\Delta \rp & = P_X\lp \hat{G}_\Delta \rp + \hat{c}P_X\lp \partial \hat{G}_{-1}\cup \partial \hat{G}_1\rp \\
        & \geq \hat{p}_1 - a_m  +  \frac{\delta - 5a_m - \hat{p}_1 + a_m}{\hat{p}_2 - \hat{p}_1 - 2a_m}\lp \hat{p}_2 - \hat{p}_1 - 2a_m\rp \\
        & \geq \delta - 3a_m. 
    \end{align*}
\end{itemize}
Thus combining the two cases, we always have $P_X\lp \hat{g}(X) = \Delta \rp  \geq \delta - 5a_m$. 
\end{proof}
%-------------------------------------------------------------------------------

%-------------------------------------------------------------------------------

Applying the lower-bound on $P_X(\hat{g}(X) = \Delta)$ from Lemma~\ref{lemma:risk1}, along with the fact that $P_X(\Gds) = \delta$, we may write
% \begin{align*}
% P_X\lp \hat{G}_{\Delta}\rp &\geq \hat{P}_m\lp \hat{G}_{\Delta}\rp - a_m \\
%     & = \delta - 2a_m 
% \end{align*}
% which, coupled with the fact that $P_X(G_{\Delta}^*) = \delta$ means that 
%
\begin{equation}
    \label{eq:risk2}
    P_X\lp G^*_{\Delta}\rp - P_X\lp \hat{g}(X)={\Delta}\rp \leq 5a_m.
\end{equation}

%------------------------------------------------------------------------------
% LEMMA: BOUNDS ON R_i and Q_i
%------------------------------------------------------------------------------

We can now upper-bound the remaining terms in~\eqref{eq:risk1}. 

\begin{lemma}
    \label{lemma:risk2}
    Assume that the events $\Omega_1$, $\Omega_2$, and $\Omega_3$ hold, and that the number of labelled samples $n$ is greater than $\zeta \coloneqq \min\{n \geq 1 \ \mid \ b_n \leq \lp \delta/2C_0 \rp^{1/ \rho_0}\}$. Then the following statements are true: 
\begin{enumerate}
\item The terms $R_i,\;i=1,\ldots,4$ satisfy
            \begin{equation}
                R_i  \leq C_04^{\rho_0 + 1} \lb \lp \frac{a_m + C_0b_n^{\rho_0}}{C_1}\rp^{(\rho_0
                +1)/\rho_1} + b_n^{\rho_0 + 1} \rb 
        \end{equation}

    \item The terms $Q_i=0,\;i=1,2$. 
\end{enumerate}
\end{lemma}
\begin{proof} 
1. We derive the required  bound  for the term $R_1$. The other terms $R_2$, $R_3$, and $R_4$ can be bounded similarly.

Using the lower-bound on $\hat{\gamma}$, we have the following:
\begin{align*}
R_1 &\leq \lp \gd - \hat{\gamma} + b_n \rp P_X \lp 1/2 - \gd \leq \eta \leq 1/2 - \hat{\gamma} + b_n \rp \\    
&\leq 5b_n P_X\lp 1/2 - \gd \leq \eta(X) \leq 1/2 - \gd + 5b_n\rp \\
& \leq C_0 (5b_n)^{1+\rho_0}. 
\end{align*}

2. We show that $Q_1=0$ by proving that the set $\hat{G}_{-1}\cap G_1^*$ is empty. The result for $Q_2$ follows similarly.
    \begin{align*}
        \hat{G}_{-1}\cap G_1^* & = \{ \hat{\eta} <1/2 - \hat{\gamma}\ , \ \eta
        > 1/2 + \gd \} \\
        & \stackrel{\text{(a)}}{\subset} \{ 1/2 +\gd < \eta < 1/2 - \hat{\gamma} + b_n\}\\ 
        & \stackrel{\text{(b)}}{\subset} \{ 1/2 + \gd < \eta < 1/2 + b_n \}.
    \end{align*}
    {\bf (a)} follows from Proposition~\ref{prop:regression_estimator}. \\ 
    {\bf (b)} uses the fact that $\hat{\gamma} \geq 0$. 
    
    Now, a necessary condition for the above set to be nonempty is that $\gd <b_n$. For $n \geq \zeta$, we can show that this is not the case. We start by 
    \begin{align*}
\delta = P_X(\Gds) &= P_X(|\eta(x) - \frac{1}{2}|\leq\gd) \leq P_X(|\eta(x) - \frac{1}{2}|\leq 2\gd) \\ 
&= P_X\lp |\eta-1/2 + \gd| \leq  \gd \rp + P_X\lp |\eta-1/2 
        - \gd| \leq \gd \rp \stackrel{\text{(a)}}{\leq} 2C_0\gd^{\rho_0}. \\
    \end{align*}
{\bf (a)} comes from applying the margin condition at levels $1/2 - \gd$ and $1/2 + \gd$.

This implies that $\gd \geq \lp \frac{\delta}{2C_0}\rp^{1/\rho_0}$, which by the assumption on $n$ being large enough ensures that $b_n\leq\gd$, and thus, the set $\hat{G}_{-1}\cap G_1^*$ is empty. 
\end{proof}

Combining these results, we obtain 
\[
    R(\hat g) - R(g^*) \leq  5a_m + 4C_0 (5b_n)^{1+\rho_0}. 
\]
as required. 
\end{proof}
%-------------------------------------------------------------------
% APPENDIX B.1
%-------------------------------------------------------------------
%-------------------------------------------------------------------
% APPENDIX B.2
%-------------------------------------------------------------------
%-------------------------------------------------------------------

%-------------------------------------------------------------------
%-------------------------------------------------------------------
% APPENDIX C
%-------------------------------------------------------------------
%-------------------------------------------------------------------

\newpage
\section{Deferred Proofs from Section~\ref{sec:convex}}

%-------------------------------------------------------------------
% APPENDIX C.1
%-------------------------------------------------------------------
% This appendix contains the proof of the excess risk bounds for the 
% classifier learned by the binary search algorithm. 

\subsection{Proof of Theorem~\ref{theorem:binary_search}}
\label{appendix:binary_search}

\begin{proofoutline}
Suppose $\lambda$ denotes the cost value at which the search algorithm stops. By the triangle inequality, it suffices to obtain separate bounds on the absolute value of the excess risks between the pairs $(\hat{g}_\lambda, g_{\lambda})$ and $(g_\lambda,g^*)$. To bound these terms, we first show that the corresponding sets of partitions of $\X$ formed by $\hat{g}_\lambda$ and $g_\lambda$ must have large overlap (in terms of $P_X$ measure) with each other (Lemma~\ref{lemma:binary1}). This allows us to obtain a lower-bound on the measure of the set $P_X\lp g_\lambda = \lambda\rp$ (Lemma~\ref{lemma:binary2}), which in turn implies that the threshold $\lambda$ is not much different from the threshold $1/2 - \gd$. These results coupled with the upper-bound on the excess surrogate (fixed-cost) risk between $\hat{g}_\lambda$ and $g_\lambda$ allow us to obtain the required bounds. 
\end{proofoutline}

\begin{remark}
\label{remark:Psi}
A concrete example of the terms $\bar{A}_n$ and $\Psi(\cdot)$  can be obtained from Corollary~19 in~\citet{yuan2010classification}. Here $\mathcal{H}$ is some class of functions $h:\X \mapsto \mbb{R}$ and 
$\mathcal{R} = \{\indi_{\{|h|>c\}}\ \mid \ h \in \mathcal{H}, \ c \in [0,\infty)\}$. If $N_n$ denotes the $1/n$ covering number of $\mathcal{H}$ w.r.t.~the uniform metric and $\varphi_\lambda(\cdot)$ is a convex surrogate satisfying the conditions of Theorem~9 in~\citet{yuan2010classification}, then we have $\bar{A}_n = \mathcal{O}\lp \frac{1}{n} + \frac{\log(n N_n)}{n}\rp$ and $\Psi(x) = x^{1/(s + \beta - s\beta)}$, for some $s>0$ and $\beta= \rho_0/(1+\rho_0)$.
\end{remark}

\paragraph{Proof of Theorem~\ref{theorem:binary_search}}. 
\begin{proof}

The choice of $\alpha_m = 2\mathfrak{R}_m(\mathcal{R}) + \sqrt{2\log(2m)/m}$
according to Claim~\ref{claim:empirical_constraint}, which ensures that with
probability at least $1-1/m$, the empirical measure and the $P_X$ measure differ
by no more than $\alpha_m$. The choice of the stopping interval is $\mathcal{I}_{n} = [\delta-3{B}_n, \delta-2{B}_n]$,  and  $B_n$ is defined in \eqref{eq:B_n}.
The algorithm stops searching in round $k$ if $Q_k \in \mathcal{I}_n$. This implies that the algorithm stops at a cost value $\lambda$, at which the fixed cost algorithm 
with $n$ labelled samples learns a classifier $\hat{g}_{\lambda}= \lp \hat{G}_{-1},
\hat{G}_{1}, \hat{G}_{\lambda} \rp$ with $\delta-3B_n - 2\alpha_m \leq
P_X(\hat{G}_{\lambda}) \leq \delta- 2B_n$.

%------------------------------------------------------------------------------
We consider the two classifiers $\hat{g}_\lambda$ which is output by the algorithm,
and $g_\lambda = \lp G_{-1}, G_1, G_\lambda \rp$ which is the optimal classifier with
cost of rejection $\lambda$. 
\begin{align*}
    \bar{R}_\lambda \lp \hat{g}_\lambda \rp - \bar{R}_\lambda \lp g_\lambda \rp  = &\int_{\hat{G}_{-1}}\eta
    dP_X + \int_{\hat{G}_1}(1-\eta)dP_X + \lambda\int_{\hat{G}_\lambda}dP_X -\\
    &  \bigg( \int_{G_{-1}} \eta dP_X + \int_{G_1}(1-\eta)dP_X + \lambda\int_{G_\lambda}
dP_X \bigg).
\end{align*}
The excess risk bound for convex surrogates of learning with fixed cost of 
abstention implies that the above term can be upper bounded by $A_n$. 
Furthermore, by proceeding as in proof of Theorem~\ref{theorem:excess_risk_bound}
we can obtain the following:

\begin{align*}
    A_n \geq &\int_{\hat{G}_{-1}\cap G_1}(2\eta -1)dP_X + \int_{\hat{G}_{-1}\cap G_\lambda}(\eta
    -\lambda)dP_X +  \int_{\hat{G}_1 \cap G_{-1}}(c-\eta)dP_X + \\ 
    &  \int_{\hat{G}_{1} \cap G_\lambda}
    (1-\eta -\lambda)dP_X + \int_{\hat{G}_\lambda \cap G_{-1}}(\lambda-\eta)dP_X + \int_{\hat{G}_\lambda
    \cap G_1}(\lambda-1+\eta)dP_X
\end{align*}

%-------------------------------------------------------------------
% $\hat{G}_c$ has a large overlap with $G_c$
%-------------------------------------------------------------------
Our next result tells us that the sets $\hat{G}_i$ have large overlap
in terms of $P_X$ measure with the sets $G_i$ for $i=-1,1$ and $\lambda$. 

\begin{lemma}
\label{lemma:binary1}
For $i=-1,1$ and $\lambda$, we have $P_X\lp \hat{G}_i \cap G_i^c \rp \leq B_n$.
\end{lemma}

\begin{proof}
We partition the set $G_{-1}$ as $G_{-1} = G_{-1,a}\cup G_{-1,b}$ where 
$G_{-1,a} = \{x \in G_{-1}\; \mid \; \eta(x) \geq \lambda-\epsilon\}$ for some 
$\epsilon>0$ to be decided later. Using this, we proceed as follows:
\begin{align*}
A_n &\geq \int_{\hat{G}_\lambda \cap G_{-1}}(\lambda-\eta)dP_X \geq \int_{\hat{G}_\lambda \cap G_{
    -1,b}}(\lambda-\eta)dP_X \\
    &\geq \epsilon P_X\lp \hat{G}_\lambda \cap G_{-1,b} \rp. 
\end{align*}
Assume that the cost $\lambda \in [1/2-\gd -\epsilon_0, 1/2-\gd + \epsilon_0]$.
We can now upper bound the probability mass of the intersection of $\hat{G}_\lambda$
with $G_{-1}$ as follows:
\begin{align*}
P_X\lp \hat{G}_\lambda\cap G_{-1}\rp & \leq P_X\lp G_{-1,a}\rp + P_X\lp \hat{G}_\lambda \cap 
G_{-1,b}\rp \\
& \stackrel{(a)}{\leq } 2\bigg(C_0\epsilon^{\rho_0} + \frac{A_n}{\epsilon}\bigg).
\end{align*}
where $(a)$ follows from the assumption that $\lambda\in [1/2-\gd-\epsilon_0, 1/2-\gd + \epsilon_0]$
and \ref{assump:margin}.
By choosing $\epsilon = \lp \frac{A_n}{C_0}\rp^{1/(\rho_0+1)}$, we get 
\begin{equation}
    \label{eq:B_n}
P_X\lp \hat{G}_\lambda \cap G_\lambda^c \rp \leq 4C_0\lp
\frac{A_n}{C_0}\rp^{\rho_0/(\rho_0+1)} 
\coloneqq B_n.
\end{equation}

Proceeding similarly, we can obtain the following bounds as well:
\begin{align*}
    P_X\lp \hat{G}_1 \cap G_1^c \rp &\leq B_n \\
    P_X\lp \hat{G}_{-1}\cap G_{-1}^c \rp &\leq B_n 
\end{align*}
\end{proof}
%-------------------------------------------------------------------
% P_X(G_c) is close to \delta
%-------------------------------------------------------------------
We now show that that $P_X(G_\lambda)$ is close to $\delta$. 
\begin{lemma}
\label{lemma:binary2}
We have $\delta - K_{m,n} \leq P_X\lp G_\lambda \rp \leq \delta$, where $K_{m,n} = \alpha_m + \frac{5}{2}B_n$.
\end{lemma}

\begin{proof}
For getting the upper bound, we use the stopping rule of the algorithm, 
and the results of Lemma~\ref{lemma:binary1}. 
\begin{align*}
    P_X(G_\lambda) &= P_X(G_\lambda\cap \hat{G}_\lambda) + P_X(G_\lambda \cap \hat{G}_\lambda^c) \\
             &\leq P_X(\hat{G}_\lambda) + P_X(\hat{G}_{-1}\cap G_\lambda) + P_X(\hat{G}_1
    \cap G_\lambda) \\
    &\leq P_X(\hat{G}_\lambda) + P_X(\hat{G}_{-1}\cap G_{-1}^c) + P_X(\hat{G}_1\cap G_1^c)\\
    & \leq \delta -2B_n + 2B_n = \delta.
\end{align*}
In the last inequality, we use $P_X(\hat{G}_{\lambda}) \leq \delta - 2B_n$ due to the stopping rule, and
$P_X(\hat{G}_i\cap G_i^c) \leq B_n$ for $i=-1,1$ from Lemma~\ref{lemma:binary1}. 
Similarly, we also have the following lower bound:
\begin{align*}
    P_X(G_\lambda) &\geq P_X(\hat{G}_\lambda \cap G_\lambda) = P_X(\hat{G}_\lambda) - P_X(\hat{G}_\lambda\cap 
    G_\lambda^c) \\
    & \geq \delta - 2\alpha_m - 3B_n -2B_n \\
    & \coloneqq \delta - 2K_{m,n}.
\end{align*}
\end{proof}
%-----------------------------------------------------------------------------
% c is close to $1/2 - \gd$. 
%-----------------------------------------------------------------------------
\begin{lemma}
\label{lemma:binary3}
Assume that the detectability assumptions \ref{assump:detect} hold with some 
$\epsilon_0>0$. Then we have $\lambda \leq \lp \frac{1}{2} + \gd \rp + 2\lp 
\frac{K_{m,n}}{C_1}\rp^{1/\rho_1}$. 
\end{lemma}

\begin{proof}

The proof of this statement relies on the fact that $G_\lambda$ and $G^*_{\Delta}$ 
are both sub-level sets of the function $|\eta - 1/2|$. From Lemma~\ref{lemma:binary2},
we know that a lower bound on  $P_X$ measure of $G_\lambda$ is $\delta - 2K_{m,n}$. 
Now, from our assumption that $|\lambda - 1/2 + \gd|\leq \epsilon_0$, and the 
detectability assumption \ref{assump:detect}, we have 
\begin{align}
    2K_{m,n} & \geq P_X\lp \Gds \rp - P_X\lp G_\lambda \rp \\
    & \stackrel{(a)}{\geq}2 C_1 \lp \frac{\lambda - 1/2 + \gd}{2} \rp^{\rho_1} \\
\end{align}
where {\bf (a)} follows from the detectabilty assuption applied at level $(1/2)(
\lambda + 1/2 - \gd)$. On simplification, this gives us
\begin{equation}
 \lp  \lambda - \frac{1}{2} + \gd \rp \leq 2\lp
    \frac{K_{m,n}}{C_1}    \rp^{1/\rho_1} \label{eq:bound_on_lambda}
\end{equation}
as required. 
\end{proof}
%-----------------------------------------------------------------------------
% The risks of $\hat{g}$ and $g^*$ are close
%-----------------------------------------------------------------------------
We now proceed towards bounding  the excess risk of the classifier output by
the binary search algorithm, $\hat{g}$, over $g^*$. We first decompose the 
excess risk into two terms. 
\begin{align}
    |R(\hat{g}) - R(g^*)| \leq |R(\hat{g})-R(g_\lambda)| + |R(g_\lambda)-R(g^*)|.
    \label{eq:binary_risk1}
\end{align}

The second term in \eqref{eq:binary_risk1} can be upper bounded as follows:
\begin{align*}
    |R(g_\lambda)-R(g^*)| &\leq \int_{G_\lambda\cap G^*_{-1}}\eta dP_X + 
    \int_{G_\lambda\cap G^*_1} (1-\eta)dP_X\\
    &\leq \int_{G_\lambda\cap G^*_{-1}}\lambda dP_X + \int_{G_\lambda \cap G^*_{-1}}(1-1+\lambda)dP_X \\
    &= \lambda P_X\lp G_\lambda \cap \big( G_{\Delta}^*\big)^c \rp \\
    &= \lambda\lp \delta - P_X(G_\lambda) \rp \\
%    &\leq cK_{m,n}\\
%    &\leq \lp \frac{1}{2} - \gd \rp K_{m,n}  + \frac{K_{m,n}^{1+1/\rho_1}}
%    {\underbar{C}^{1/\rho_1}}
\end{align*}

Now, for the first term, we have
\begin{align*}
    |R(\hat{g}) - R(g_\lambda)| &\leq A_n + \lambda|P_X(\hat{G}_\lambda) - P_X(G_\lambda)| \\
                          &\leq A_n + \lambda(\delta - P_X(G_\lambda))\\
                          &\stackrel{(a)}{\leq} A_n + \lp \frac{1}{2} - \gd + \lp \frac{K_{m,n}}
    {C_1}\rp^{1/\rho_1}\rp 2K_{m,n}\\
\end{align*}
where {\bf (a)} follows from Lemma~\ref{lemma:binary3}. 
Combining these inequalities, we get the required bound on the excess risk of 
the classifier $\hat{g}$.   
\begin{align*}
    R(\hat g) - R(g^*) & \leq A_n + 2\lambda (\delta - P_X(G_\lambda))\\
                       & \leq A_n + 4\lp \frac{1}{2}-\gd\rp K_{m,n} +
4\frac{K_{m,n}^{1+1/\rho_1}}{C_1^{1/\rho_1}}.
\end{align*}

%------------------------------------------------------------------------------
% LAMBDA LIES IN THE REQUIRED INTERVAL
%------------------------------------------------------------------------------
It remains to show that the assumption that $|\lambda - 1/2 + \gd| < \epsilon_0$ is 
satisfied. 

\begin{lemma}
\label{lemma:binary4}
Suppose $n$ and $m$ are large enough to ensure that  $2\alpha_m + 3B_n \leq 2C_1(
\epsilon_0/8)^{\rho_1}$, and $A_n < 2C_1\lp \epsilon_0/4\rp^{1+\rho_1}$. Then   $|\lambda - 1/2 + \gd|$ is smaller than  $\epsilon_0$. 
\end{lemma}

\begin{proof}
    We proceed by contradiction. Assume that $\lambda>1/2 - \gd + \epsilon_0$. 
    (The case $\lambda<1/2 - \gd - \epsilon_0$ can be handled similarly). 
Let $\lambda_1$ denote the threshold at which we have $P_X(|\eta - 1/2|\leq \lambda_1) =
\delta - 3B_n - 2\alpha_m$, and let $G_{\lambda_1} = \{x \in \X\ \mid |\eta(x) - 1/2|
\leq \lambda_1\}$. 
By the condition on $m$ and $n$, we know that $3B_n + 2\alpha_m \leq 2C_1\lp \epsilon_0/8
\rp^{\rho_1}$, which implies that $\lambda_1 \leq 1/2 - \gd + \epsilon_0/4$. 

Define the set $U = \{ x \in \X \ \mid \ |\eta -1/2 +\gd - \epsilon_0/2|\leq 
\epsilon_0/4\}$.  By our assumption that $\lambda > 1/2 - \gd + \epsilon_0$,
the set $U$ is a subset of  $G_{\lambda_1}\setminus G_\lambda$, and 
for all $x \in U$, we have $|\eta(x)-\lambda| \geq \epsilon_0/4$.
Furthremore, by applying the detectability condition at level $1/2 - \gd + \epsilon/2$,
we have $P_X(U) \geq 2C_1\lp \frac{\epsilon_)}{4}\rp^{\rho_1}$. 

We now observe the following:
\begin{align*}
    A_n &\geq \bar{R}_\lambda\lp \hat{g}_\lambda \rp - \bar{R}_\lambda(g_\lambda)
    \geq \int_{\hat{G}_\lambda \setminus G_\lambda}|\eta - \lambda|dP_X \\
    &\stackrel{(a)}{\geq} \int_{G_{\lambda_1}\setminus G_\lambda}|\eta - \lambda|dP_X 
      \geq \int_U |\eta - \lambda|dP_X \\
    & \stackrel{(b)}{\geq} \frac{\epsilon_0}{4}C_1 \lp \frac{\epsilon_0}{4}\rp^{\rho_1}
     = 2C_1\lp \frac{\epsilon_0}{4}\rp^{1+\rho_1} > A_n,
\end{align*}
which gives us the required contradiction. In the above display,
{\bf (a)} follows from the fact that $P_X(G_{\lambda_1}\setminus G_\lambda) \leq
P_X(\hat{G}_\lambda \setminus G_\lambda)$, and an argument similar to the proof
of Proposition~\ref{prop:bayes1}.
{\bf (b)} follows from the results on $P_X(U)$ and $|\eta(x) - \lambda|$ for $x\in U$. 

\end{proof}

\end{proof}

\subsection{Slack Term in~\eqref{convex_surrogate_2}}
\label{appendix:empirical_constraint}
We present a result which provides us with an appropriate value of slack which ensures that the constraint in~\eqref{convex_surrogate_2} is satisfied with high probability. 

\begin{proposition}
\label{prop:empirical_constraint}
Let $\bar{\varphi}_H(z) \coloneqq \min \{1, \varphi_H(z)\}$ be the clipped version of the hinge loss and define $\mathcal{F} \coloneqq \{\bar{\varphi}_H\circ r \; \mid \; r \in \mathcal{R}\}$. Then, for any $m >1$ and all $r\in\mathcal{R}$, with probability at least $1-1/m$, we have
% \todo{I use $j$ for the index of training samples. This is because we have used $i$ or $i_h$ as a function. Make sure we are consistent everywhere, especially in the appendix.}
%
\begin{equation*}
P_X \big( r(X) \leq 0 \big) \leq \frac{1}{m}\sum_{j=n+1}^{n+m} \varphi_H\big( r(X_j)\big) + \frac{\tau}{\sqrt{m}}, 
\end{equation*}
where $\tau = 2\sqrt{m}\mathfrak{R}_m\lp \mathcal{F}\rp + \sqrt{2\log(2m)}$ and $\mathfrak{R}_m(\mathcal{F})$ is the Rademacher complexity of the function class $\mathcal{F}$. 
\end{proposition}
\begin{proofoutline}
The proof follows by employing the standard Rademacher complexity generalization bound~\citep[Theorem~26.5]{shalev2014understanding} over the bounded class of functions $\mathcal{F}$, and then using the fact that $\varphi_H\circ r \geq \bar{\varphi}_H \circ r$.  The detailed proof is given in Appendix~\ref{appendix:empirical_constraint}.
\end{proofoutline}

\begin{proof}
We proceed as follows
\begin{align*}
    \mbb{E}\lb \bar\varphi_H\big(r(X)\big)\rb\big) - 
     \frac{1}{m}\sum_{i=1}^m \bar \varphi_H\big( r(X_{i+m})\big)& \stackrel{(a)}{\leq}
    2\mathfrak{R}_m(\mathcal{F}) + \sqrt{\frac{2\log(2m)}{m}} \\
\Rightarrow     \mbb{E}\lb \bar\varphi_H\big(r(X)\big)\rb - 
 \frac{1}{m}\sum_{i=1}^m\varphi_H\big( r(X_{n+i}) \big)& 
\stackrel{(b)}{\leq } 2\mathfrak{R}_m(\mathcal{F}) + \sqrt{\frac{2\log(2m)}{m}} .
\end{align*}
The inequality $(a)$ in the above display follows form the Rademacher
complexity generalization bounds for the bounded loss function $\bar
\varphi_H(\cdot)$, while $(b)$ follows from the fact that $\varphi_H(x) \geq
\bar \varphi(x)$ for all $x$. Finally, the result is obtained by using the
fact that $P(r(X) \leq 0) \leq \mbb{E}\lb \bar\varphi_H\big(r(X)\big)\rp$.

\end{proof}
%-------------------------------------------------------------------
% APPENDIX C.2
%-------------------------------------------------------------------

%-------------------------------------------------------------------
%-------------------------------------------------------------------
%-------------------------------------------------------------------
%-------------------------------------------------------------------
%-------------------------------------------------------------------
% APPENDIX: EXPERIMENT DETAILS
%-------------------------------------------------------------------
\newpage 
\section{Details of Experiments}

\subsection{Details of Implementation}
\label{appendix:details_experiment}
\paragraph{Algorithm~1.} For the fixed cost subroutine required
by Algorithm~1, we implemented the primal form of the CHR algorithm of \citep[\S~4]{cortes2016learning}, employing  the random Fourier feature maps of
\citep{rahimi2008random} of RBF kernel. For selecting the regularization parameters
of the objective function of CHR algorithm, we performed a grid search over the set
$\{10^i\ \mid \ -5 \leq i\leq 5\}$. We set the slack term $\alpha_m = 0.1/\sqrt{m}$ 
and the algorithm stopped when $\delta - Q_k \leq \text{tol}$, and we used the value 
$\text{tol}=0.01$.

\vspace{-1em}
\paragraph{Algorithm~2.} We used the hinge loss
$\varphi_H(\cdot)$ for both the objective and the constraint. This
however, resulted in very conservative solutions for which the abstention 
rate was much smaller than $\delta$ due to the more stringent constraints.
To alleviate this problem, 
we relaxed the constraint by using the term $c \delta$ for $c\in [1,2]$, and the value of $c$ was chosen by grid search.

We now report the figures for three other benchmark machine learning datasets, namely \texttt{cod-rna}, \texttt{skin} and \texttt{digits}.

\begin{figure}[h]
\label{fig:cod}
\centering
\includegraphics[ width=0.9\linewidth]{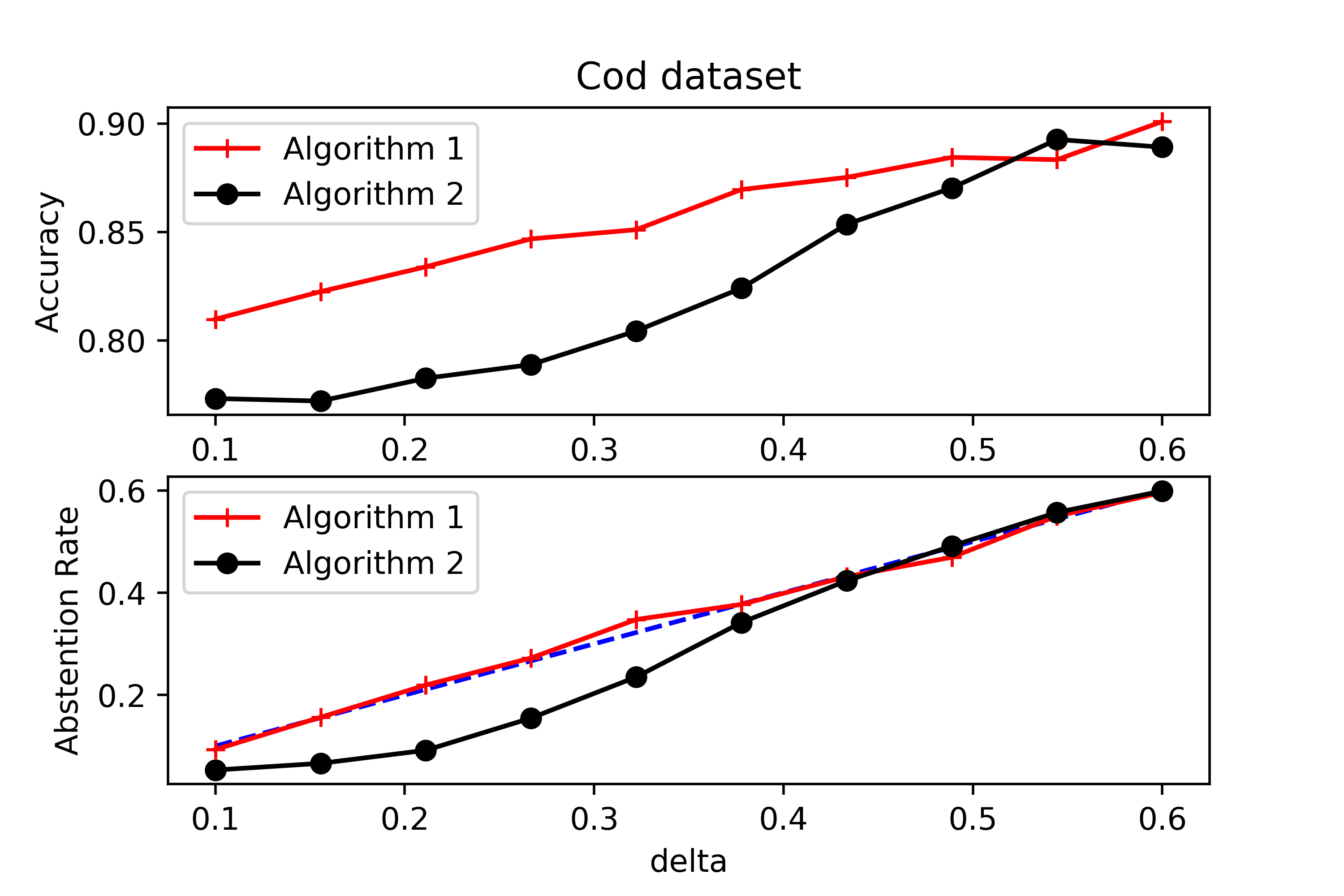}
\caption{Plot of the rejection rate versus accuracy as $\delta$ varies from $0.1$
to $0.6$ for the two algorithms on the \texttt{cod-rna} dataset.}
\end{figure}

\begin{figure}[h]
\label{fig:digits}
\centering
\includegraphics[ width=0.9\linewidth]{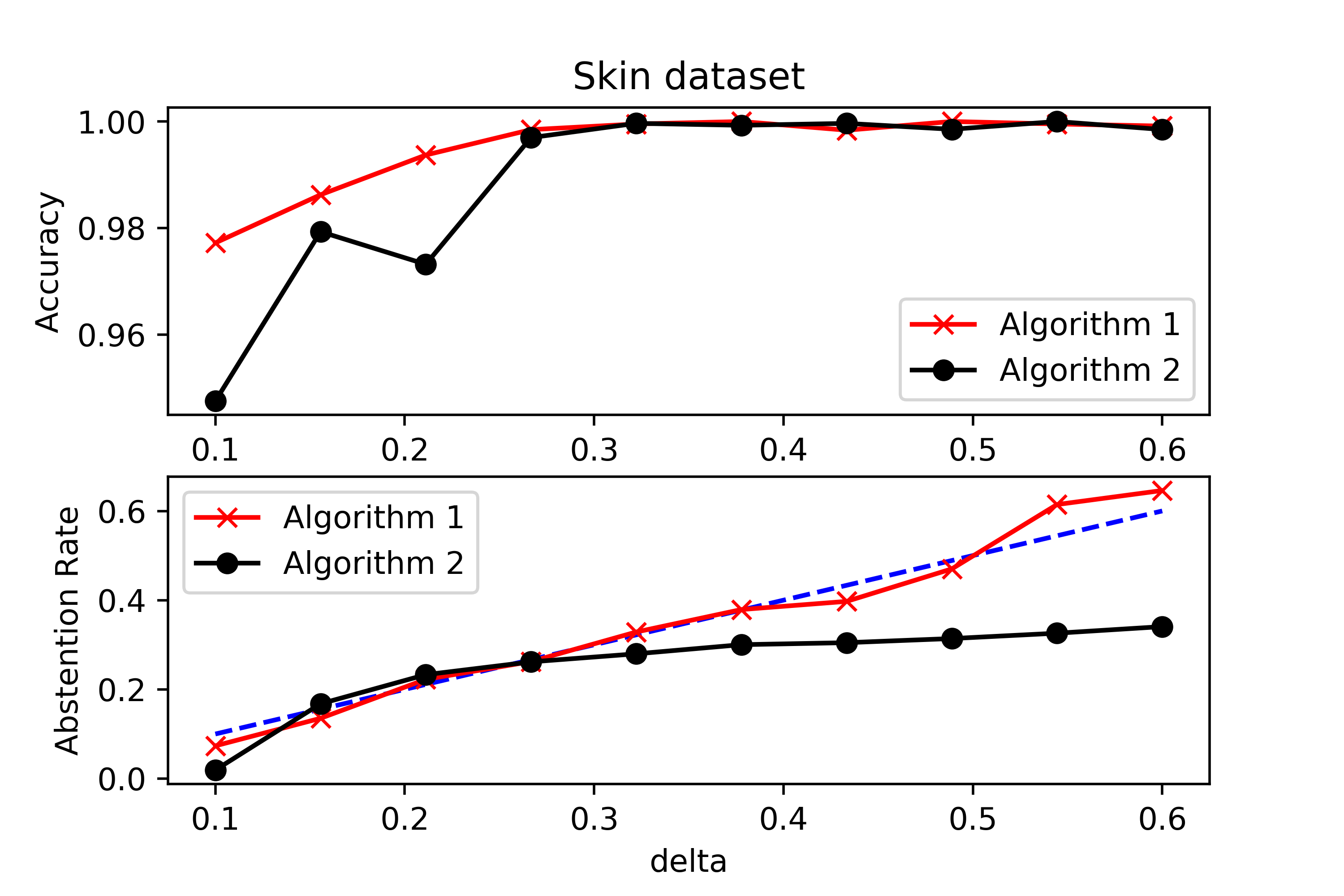}
\caption{Plot of the rejection rate versus accuracy as $\delta$ varies from $0.1$
to $0.6$ for the two algorithms on the \texttt{skin} dataset.}
\end{figure}

\begin{figure}[h]
\label{fig:digits}
\centering
\includegraphics[ width=0.9\linewidth]{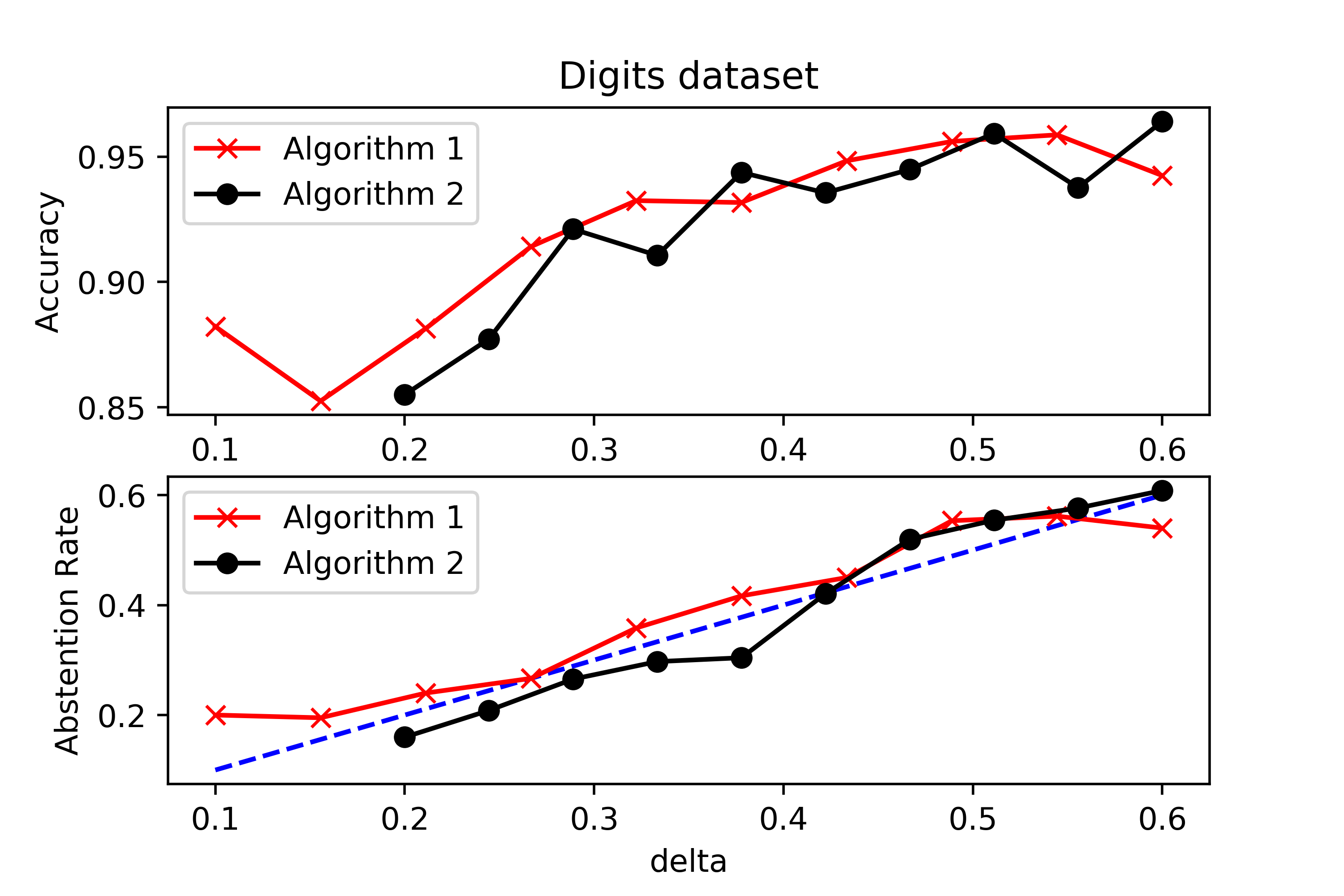}
\caption{Plot of the rejection rate versus accuracy as $\delta$ varies from $0.1$
to $0.6$ for the two algorithms on the digits dataset.}
\end{figure}

\label{appendix:experiments}

\end{appendix}

%-------------------------------------------------------------------
%-------------------------------------------------------------------
%-------------------------------------------------------------------
%-------------------------------------------------------------------
%-------------------------------------------------------------------

\end{document}